\documentclass{article}

\usepackage[preprint,nonatbib]{neurips_2024}
\usepackage{graphicx, subfigure}

\usepackage[utf8]{inputenc} 
\usepackage[T1]{fontenc}    
\usepackage{hyperref}       
\usepackage{url}            
\usepackage{booktabs}       
\usepackage{amsfonts}       
\usepackage{nicefrac}  
\usepackage{float}
\usepackage{microtype}      
\usepackage{xcolor}         
\usepackage{amsthm,amssymb}
\usepackage{amsmath,amsfonts}
\usepackage{cleveref}
\usepackage{thm-restate}
\usepackage{enumitem}
\usepackage{comment}
\usepackage[linesnumbered, ruled]{algorithm2e}

\newcommand{\R}{\mathbb{R}} 
\renewcommand{\P}[1]{\mathbb{P}\left(#1 \right)} 
\newcommand{\E}[1]{\mathbb{E}\left[ #1 \right]} 
\newcommand{\N}{\mathbb{N}} 
\renewcommand{\S}{\mathbb{S}} 

\newcommand{\vq}{\mathbf{q}}
\newcommand{\vqp}{\mathbf{q'}}

\renewcommand{\E}{\mathbb{E}}

\newcommand{\wQ}{\widehat{Q}}

\newcommand{\pbr}[1]{\left( #1\right)}
\newcommand{\sbr}[1]{\left[ #1\right]}
\newcommand{\cbr}[1]{\left\{ #1\right\}}

\makeatother

\renewcommand{\S}{\mathcal{S}}
\newcommand{\A}{\mathcal{A}}

\renewcommand{\P}{\mathbb{P}}

 \newtheorem{theorem}{Theorem}[section]

\newtheorem{lemma}[theorem]{Lemma}

\newtheorem{assumption}[theorem]{Assumption}

\DeclareMathOperator*{\argmin}{arg\,min}
\title{Performance of NPG in Countable State-Space Average-Cost RL}

\author{%
  Yashaswini Murthy \\ 
  Electrical and Computer Engineering\\
  University of Illinois Urbana-Champaign\\
  Urbana, IL 61801 \\
  \texttt{ymurthy2@illinois.edu} \\
   \And
   Isaac Grosof \\
   Electrical and Computer Engineering \\
   University of Illinois Urbana-Champaign \\
   Urbana, IL 61801 \\
   \texttt{igrosof@illinois.edu}
   \And
   Siva Theja Maguluri \\
   Industrial and Systems Engineering \\
   Georgia Institute of Technology \\
   Atlanta, GA 30332 \\
   \texttt{siva.theja@gatech.edu}
   \And
   R. Srikant \\
   Electrical and Computer Engineering \\
   University of Illinois Urbana-Champaign \\
   Urbana, IL 61801 \\
   \texttt{rsrikant@illinois.edu}
   }

\begin{document}

\maketitle

\begin{abstract}
   We consider policy optimization methods in reinforcement learning settings where the state space is arbitrarily large, or even countably infinite. The motivation arises from control problems in communication networks, matching markets, and other queueing systems. Specifically, we consider the popular Natural Policy Gradient (NPG) algorithm, which has been studied in the past only under the assumption that the cost is bounded and the state space is finite, neither of which holds for the aforementioned control problems. Assuming a Lyapunov drift condition, which is naturally satisfied in some cases and can be satisfied in other cases at a small cost in performance, we design a state-dependent step-size rule which dramatically improves the performance of NPG for our intended applications. In addition to experimentally verifying the performance improvement, we also theoretically show that the iteration complexity of NPG can be made independent of the size of the state space. The key analytical tool we use is the connection between NPG stepsizes and the solution to Poisson's equation. In particular, we provide policy-independent bounds on the solution to Poisson's equation, which are then used to guide the choice of NPG stepsizes.
\end{abstract}

\section{Introduction}

We are motivated by control problems in queueing models of resource allocation,
such as those arising in communication networks, cloud computing systems, and riding hailing services. Examples of such systems include the following:
\begin{enumerate}[leftmargin=*,label=(\alph*)]
    \item The switch fabric in Internet routers and data centers where packets have to be transported (or switched) from one of many input ports to one of many output ports \cite{srikant2013communication}: the system is modeled as a bipartite graph with input ports on one side and output ports on the other side. Technological constraints dictate that at each time slot,
    a matching must be selected in the bipartite graph,
    and packets are transferred along the edges of the matching from each input to the corresponding output.
    The goal is to find a sequence of matchings to minimize either the average delay experienced by the packets in the switch or the probability that the delay exceeds some threshold.
    \item Scheduling problems at base stations in 5G networks \cite{srikant2013communication}: at a central controller (typically the base station associated with a cell in a cellular network), packets arrive and are queued in a separate queue for each receiver.
    The goal is to schedule these packets over different frequencies and time slots to minimize the average delay of the packets in the system,
    while taking into account the time-varying channel conditions in a wireless network due to fading and other wireless medium effects. 
    \item Scheduling workloads in cloud computing systems \cite{mao2019learning}:
    a workload in such systems takes the form of a collection directed acyclic graphs,
    where each DAG represents a job, the nodes in the graphs represent tasks in the job and the directed edges represent precedence relationships among the tasks in the graph. The goal is to allocate resources to tasks from a sequence of arriving jobs,
    while respecting the precedence relationships of the tasks within each job and minimizing the average delay experienced by the jobs.
    \item Customer-driver matching in ride hailing platforms such as Uber and Lyft \cite{ozkan2020dynamic,varma2023dynamic}:  the role of such platforms can be modeled as controlling the number of nodes in a bipartite graph, where one side is the set of waiting customers and the other side in the set of available drivers. The goal of a ride hailing platform is to choose a set of prices and match customers to drivers so that a weighted combination of the average delay experienced by customers and the average profit is optimized. 
\end{enumerate}

The above problems exhibit several common features:
\begin{enumerate}[leftmargin=*,label=(\roman*)]
    \item The state space of these problems is discrete, typically consisting of the queue lengths of the various entities waiting in the system such as packets, customers, drivers, jobs and tasks, depending on the context. Discrete state spaces are commonly studied in the reinforcement learning (RL) literature; however, in our applications, the state space is also countably infinite for all practical purposes, since queue lengths can become unbounded. In some applications, such as communication networks, the packet buffers may be finite but it is well known that modeling them as infinite buffers leads to good scheduling algorithms \cite{srikant2013communication}. It should be noted that even if one were to model the finiteness of the buffers explicitly in our model, our results will still hold, and our performance guarantees would not depend on the size of the buffers.
    \item Because we are dealing with a vector of queue lengths as the state of the system, the problems have some limited amount of structure that can and should be exploited to design good algorithms. In particular, it is relatively straightforward to design algorithms that ensures that the system is stable, i.e., the queue length is finite with probability one \cite{srikant2013communication}. On the other hand, algorithms to optimize performance objectives such as average delay are unknown except in limited regimes \cite{eryilmaz2012asymptotically,maguluri2016heavy}. Therefore, data-driven approaches such as reinforcement learning (RL) are natural candidates to solve such problems. 
    \item Due to a well-known result called Little's law, minimizing average delay is equivalent to minimizing average queue lengths \cite{little2008little}. Thus, the natural instantaneous cost in such problems is the current total queue length. Note that unlike many RL models, this cost is unbounded and results which assume that the costs (or rewards) at each (state, action) pair are uniformly bounded do not hold for our problems.
\end{enumerate}

Given the above background, our goal in this paper is to study policy optimization algorithms for such countable state space models with discrete, finite action spaces where the cost is proportional to the total queue length in the system, and can thus grow in an unbounded fashion. For this purpose, we study the natural policy gradient (NPG) algorithm. Our main contributions are the following:
\begin{enumerate}[leftmargin=*,label=(\arabic*)]
    \item \textbf{Algorithmic Contribution:} A standard regret analysis for NPG relies on its connection to a classic learning theory problem known as the best-experts problem. However, we demonstrate that this analysis doesn't hold in our case due to the unbounded nature of the instantaneous cost in our setting. By making a small but crucial adjustment to the step size used in the best-experts algorithm and leveraging bounds on the relative value function (also known as the solution to Poisson's equation in applied probability), we establish nontrivial regret bounds. In addition to ensuring the convergence of NPG in countable state-space MDPs, this algorithmic adjustment significantly accelerates convergence in finite state MDPs compared to the fixed step-size NPG algorithm. Notably, prior work offered no heuristic for selecting an optimal fixed step size, often relying on hyperparameter tuning. Our approach, grounded in Poisson's equation, provides an effective heuristic for selecting both fixed and adaptive step sizes without the need for extensive parameter tuning. This improvement streamlines the application of NPG and enhances its overall performance.
    \item \textbf{Theoretical Contribution:} An important component of our work is to obtain bounds on the solution of Poisson's equation that are uniform across all policies. To the best of our knowledge, prior works on obtaining bounds on the solution of Poisson's equation are limited to specific policies. A key contribution of our paper is to show that uniform bounds can be obtained by exploiting certain structural properties of the mathematical models for the motivating applications mentioned earlier. \textcolor{black}{These bounds are essential for achieving final regret bounds that are independent of the state space cardinality.} 
    \item \textbf{Relaxation of learning error assumption:} Policy evaluation using temporal difference learning and Monte Carlo methods have been well studied in the literature, so we do not consider them explicitly in this paper. However, we do consider the error due to function approximation. Traditionally, for analytical purposes, it is assumed that there is a uniform bound on the function approximation error of the value function. \textcolor{black}{We argue that this assumption is not reasonable for countable state space models, particularly queuing models with unbounded instantaneous costs. Instead, we propose a more general model for function approximation, where the error bounds in learning are relaxed for states that are less frequently visited.} Existing mathematical tools for the study of convergence of RL algorithms cannot handle our proposed model for the function approximation error. However, we show that, by exploiting the special structure of our queueing models and the associated bounds on the solution to Poisson's equation, we can obtain non-trivial regret bounds for policy optimization.
    \item \textbf{Empirical Evaluation:}  \textcolor{black}{We evaluate the performance of our algorithmic modification in finite state space applications, focusing specifically on cloud computing scenarios driven by autoscaling. We conduct two sets of experiments where TD($\lambda$) is used to learn the value functions. 
    Utilizing our bounds on the solution to Poisson's Equation, we determine the fixed step size according to established theory and compare the performance of our algorithm against one that employs this fixed step size. Additionally, we conduct experiments in a noiseless environment to evaluate the robustness of our algorithm against learning errors. We empirically show the vast improvement in convergence when utilizing an adaptive state dependent step size to that of a fixed step size, where the former rate of convergence remains independent of the underlying state space cardinality. By demonstrating similar iteration complexities with and without noise, we validate the robustness of our algorithm within our framework of relaxed assumptions over learning error, which accommodates greater noise in the value function for less frequently visited states.}
\end{enumerate}

\subsection{Related Work}

The Natural Policy Gradient algorithm is a well-known and extensively studied algorithm for MDP optimization, in both the average-reward and discounted-reward settings \cite{kakade2001natural,agarwal2021theory,geist2019theory,murthy2023convergence,murthy2023performance}.
An important line of research on the NPG algorithm treats the MDP-optimization problem as many parallel instances of the expert advice problem, and treats the NPG algorithm as many parallel instances of the weighted majority algorithm.
Even-dar et al. \cite{even2009online} use this approach to prove the first convergence result for NPG in the finite-state average-reward tabular setting, and \cite{abbasi2019politex} expand upon that result to incorporate function approximation.
Our result uses the same ``parallel weighted majority'' framing, but generalizes the result to the infinite-state-space setting by incorporating state-dependent learning rates.

Policy gradient algorithms have been studied in certain specialized settings with average-reward uncountably-infinite state spaces \cite{fazel2018global,kunnumkal2008using}: the Linear Quadratic Regulator and the base-stock inventory control problem, demonstrating rapid convergence to the optimal policy.
However, follow-up study of these settings has demonstrated that they exhibit additional structure which is critical to these results, causing these policy-gradient algorithms to act like policy improvement algorithms \cite{bhandari2024global}.
Our result is the first to handle an infinite state-space setting without the specialized structure of these prior results.

Key to our result are novel bounds on the relative value function, building off of our drift assumption for the policy space. This drift assumption is reasonable in a queueing setting, as we discuss in \cref{disc:assumptions} \cite{srikant2013communication}.
Prior drift-based bounds on the relative-value function exist \cite{glynn1996liapounov}, but are policy-dependent. In contrast, we prove policy-independent bounds on the relative-value function using reasonable assumptions on the MDP structure,
which are motivated by the structure of MDPs in queueing networks. Our policy-independent bounds are critical to implement our state-dependent learning rates, allowing us to generalize the NPG algorithm to the infinite-state setting.

A variety of papers have studied applications of reinforcement learning to queueing problems, including policy-gradient-based algorithms.
Several such results focus on the problem of learning the relative value function from samples, including variance reduction techniques \cite{dai2022queueing} and sample augmentation techniques \cite{wei2023sample}.
Our results complement these results, as we focus on the function approximation step,
and prove results on overall algorithmic performance,
while these results focus on the policy evaluation step,
and empirically demonstrate performance improvements.
Dai and Gluzman \cite{dai2022queueing} in particular empirically demonstrate that with variance reduction techniques in use, policy gradient algorithms with function approximation rapidly converge to the optimal policy in an infinite-state-space queueing setting. Our results theoretically justify this empirical observation.

\section{Model and Preliminaries}
\textcolor{black}{
We consider the class of Markov Decision Processes (MDP) with countably infinite states $\S$, finite actions $\A$ and the infinite horizon average cost objective. We consider a randomized class of policies $\Pi,$ where a policy $\pi\in\Pi$ maps each state to a probability vector over actions $\A$, that is, $\pi:\S \to\Delta\A$. The state and action at time $t$ are denoted by $(\vq_t,a_t)$ respectively. The underlying probability transition kernel is denoted by $\P:\S\to\S$ and the transition kernel corresponding to any policy $\pi$ is denoted by $\P_\pi,$ where $\P_\pi(\mathbf{q'}|\vq) = \sum_{a\in\A}\pi(a|\vq)\P(\mathbf{q'}|\vq,a)$ is the probability of transitioning from $\vq$ to $\vqp$ under policy $\pi$ in a single step. Associated with each state $\vq$ and action $a$ is a single step cost $c(\vq,a)$ which is non-negative. Let $\underline{c}(\vq)=\min_{a\in\A} c(\vq,a)$ and $\overline{c}(\vq)=\max_{a\in\A}c(\vq,a)$ be the minimum and maximum instantaneous cost respectively in state $\vq$ across all actions $a\in\A$. The single step cost under policy $\pi$ at state $\vq$ is thus denoted by $c_\pi(\vq)$, where $c_\pi(\vq) = \sum_{a\in\A}\pi(a|\vq)c(\vq,a).$ When dealing with queuing systems this single step cost can correspond to total queue length, which can be unbounded thus yielding unbounded instantaneous costs. Hence this formulation relaxes the bounded single step costs assumption which is a common feature of many algorithms previously studied in literature \cite{abbasi2019politex,murthy2023performance}.} The infinite horizon average cost associated with a policy $\pi$ is denoted by $J_\pi$, and is defined as follows:
\begin{equation}
    J_\pi = \lim_{T\to\infty}\frac{\E_\pi\sbr{\sum_{t=0}^{T-1}c_\pi(\vq_t)}}{T}
\end{equation}
where the expectation is taken with respect to the trajectory generated by $\P_\pi.$ If the transition kernel $\P_\pi$ admits a unique stationary distribution $d_\pi$ over the state space, then the infinite horizon average reward can be reformulated as $ J_\pi = \sum_{\vq\in\S}d_\pi(\vq)c_\pi(\vq)$.
If a function $V_\pi:\S\to\R$ associated with a policy $\pi$ is absolutely integrable, that is it satisfies:
\begin{equation}
    \sum_{\vqp\in\S}\P_\pi(\vqp|\vq)|V_\pi(\vqp)|<\infty
\end{equation}
and is a solution to the Poisson's equation:
\begin{equation}
\label{poissons}
    J_\pi + V_\pi(\vq) = c_\pi(\vq)+ \sum_{\vqp\in\S}\P_\pi(\vqp|\vq)V_\pi(\vqp), \qquad \forall \vq\in\S
\end{equation}
then $V_\pi(\vq)$ is defined as the relative value function associated with the policy $\pi$ \cite{glynn2024solution}. Since $V_\pi(\vq)$ is unique upto an additive constant, any function of the form $V_\pi(\vq)+C$, where $C$ is a constant is also a solution to the Poisson's equation. However, the most frequently used representation of the value function, which is also unique, is given by:
\begin{equation}
\label{eq:valfunc_rec}
    V_\pi(\vq)=\E_\pi\sbr{\sum_{i=0}^{\tau^\pi_{\vq^*}-1}\pbr{c_\pi(\vq_i)-J_\pi}\Bigg|\vq_0=\vq}
\end{equation}
where $\tau^\pi_{\vq^*}$ represents the first time to hit state $\vq^*$ starting from any state $\vq$ under policy $\pi.$ Hence, from definition it follows that $V_\pi(\vq^*)=0.$ The value function associated with a state $\vq$ represents the expected difference between the total cost and the expected total cost obtained under policy $\pi$ when starting from state $\vq$ until state $\vq^*$ is reached for the first time.  
The relative state action value function $Q_\pi(\vq)$ is analogously defined as the solution to the following equation:
\begin{equation}
\label{eq:Qfunc}
    J_\pi + Q_\pi(\vq, a) = c(\vq,a) + \sum_{\vqp\in\S}\P(\vqp|\vq,a)V_\pi(\vqp)
\end{equation}
The state action value function $Q_\pi(\vq,a)$ has a similar interpretation as state value function $V_\pi(\vq)$ except the action enacted at time $0$ is $a$ and not dictated by the policy $\pi.$ 

The goal of reinforcement learning is to determine the policy $\pi^*\in\Pi,$ such that the infinite horizon average cost is minimized. That is, to solve for 
\begin{equation}
   J^* = \min_{\pi\in\Pi} J_\pi
\end{equation}
where $\pi^* = \argmin_{\pi\in\Pi} J_\pi$.
The focus of this paper is to analyze the performance of Natural Policy Gradient in determining the optimal policy that minimizes the infinite horizon average cost. 

\subsection{Natural Policy Gradient Algorithm}

Natural Policy Gradient algorithm is related to the mirror descent algorithm in the context of tabular policies. The objective of mirror descent involves minimizing the first order approximation of the average cost with $\mathsf{KL}$ regularizer. In the context of tabular policies, the NPG policy update we consider is of the form below:
\begin{equation}
\label{eq:upd}
    \pi_{i+1}(a|\vq) \propto \pi_i(a|\vq)\exp{\pbr{-\eta_\vq \widehat{Q}_{\pi_i}(\vq,a)}}
\end{equation}
where $\eta_\vq>0$ is the state dependent step size and $\widehat{Q}_{\pi_i}$ is the estimate of state action value function $Q_{\pi_i}$ learnt using policy evaluation algorithms. Since in the limit as $\eta_\vq\to\infty$, the above update picks the action with the lowest state action value function, NPG is also considered to be a form of soft policy iteration. The magnitude of $\eta_\vq$ determines the greediness of the policy. 

\textcolor{black}{In finite state spaces, previous literature generally considers a state-independent step size, $\eta$ {\cite{abbasi2019politex,murthy2023performance}}, where theoretical convergence guarantees require $$\eta \leq \min_{\substack{\pi\in\Pi\\\vq\in\S,a\in\A}}\frac{1}{\widehat{Q}_{\pi}(\vq,a)}.$$ However, as we will demonstrate, in the context of unbounded instantaneous costs, such as those encountered in queuing systems with infinite buffers, the value of $Q_\pi$ and its estimate $\widehat{Q}_{\pi}$ increases with the cardinality of the state space. This creates two issues with using a fixed step size: (i) Even in the context of finite state spaces, there are currently no clear guidelines for selecting this $\eta$, leading most practical applications to rely on hyperparameter tuning to find a value of $\eta$ that achieves some level of convergence; (ii) As the state space grows larger (even if finite), the value of $\eta$ required for algorithm convergence decreases, resulting in a sharp increase in the iteration complexity of NPG. In the following sections, we address how both of these issues can be mitigated by identifying policy-independent bounds on $\widehat{Q}_\pi$ and using them to determine a state-dependent step size $\eta_\vq$. This approach yields an NPG iteration complexity that is independent of the state space cardinality, providing non-trivial convergence bounds for countable state spaces as well.}

With the state space being infinitely large, a common approach to evaluate value functions is through linear function approximations. This simplifies the complexity from infinity to the dimension of the parameter vector, although with some loss in accuracy. A popular method involves using neural networks, where the weights act as the parameter vector and the network itself serves as the feature space. For each policy $\pi$, the estimate $\widehat{Q}_{\pi}$ of the state-action value function $Q_\pi$ is then computed using overparametrized neural networks and samples gathered from trajectories under policy $\pi$. For further details, please see Subsection~\ref{subsubsec:Evaluation}.


\IncMargin{1 em}
\begin{algorithm}[H] \label{alg: NPG}
\DontPrintSemicolon
\SetAlgoNlRelativeSize{-1}
\caption{Natural Policy Gradient Algorithm}
    \SetKwInOut{Inputs}{Require}
    \BlankLine
    \Inputs{ $T$, $\pi_0\in\Delta\A$}
    \BlankLine
    \For{$i = 0,1,2,3, \cdots, T-1$}{
            Generate trajectory $\cbr{\vq_0,a_0,\vq_1,a_1,\ldots,\vq_n,a_n}$ using policy $\pi_i$. Evaluate $\widehat{Q}_{\pi_i}$ using neural network linear function approximation.\;
            Update policy as:\;
            \begin{equation}
            \label{eq:update}
                \pi_{i+1}(a|\vq)=\frac{\pi_i(a|\vq)\exp{\pbr{-\eta_\vq \widehat{Q}_{\pi_i}(\vq,a)}}}{\sum_{a'\in\A}\pi_i(a'|\vq)\exp{\pbr{-\eta_\vq \widehat{Q}_{\pi_i}(\vq,a')}}}
            \end{equation}
    }
    \textbf{return} $\pi_T$
\end{algorithm}

To aide our analysis, we make the following assumptions, which are typically met by queuing systems. The irreducibility of the Markov chain under any policy is a standard assumption in reinforcement learning. This ensures adequate exploration and visitation of all state-action pairs, which is crucial for learning policies with reasonable confidence.  
\begin{assumption}
\label{asspn:asump1}
    For all policies $\pi\in\Pi,$ the induced Markov Chain $\P_\pi$ is irreducible.
\end{assumption}
In countable state Markov chains, irreducibility together with positive recurrence ensures the existence of the stationary distribution which aides in the proof of convergence of NPG. The next assumption ensures that 
the underlying Markov chain is positive recurrent (see Lemma~\ref{lem: positive_recurrence}).

\begin{assumption}
\label{asspn:asump2}
\textcolor{black}{
There exists a function $f:\S\to[0,\infty)$ and constants $\epsilon>0, g, D$ independent of policy $\pi$ such that for every policy $\pi\in\Pi$ and every $\vq\in\S,$ 
\begin{enumerate}
    \item The drift equation
    \begin{equation}
    \label{eq:drift}
        \E_\pi\sbr{f^2(\vq_{k+1})-f^2(\vq_k)|\vq_k=\vq}\leq -\epsilon\overline{c}(\vq)+g.
    \end{equation}
    is satisfied.
    \item Single step transitions are uniformly bounded, i.e.,
    \begin{equation}
    \label{eq:bdddiff}
    \left|f(\vq')-f(\vq)\right|\leq D \quad \forall \vq'\in\S: \P_\pi(\vq'|\vq)>0.
    \end{equation}
    \item The set 
        \begin{equation}
        \label{eq: set_B}
            B:= \cbr{\vq\in\S: \underline{c}(\vq)\leq \frac{2g}{\epsilon}},
        \end{equation}
        is finite and $f(\vq)>0$ if $\vq\in B^c.$
\end{enumerate}
}
\end{assumption}
We will call the function $f$ in the above assumption a Lyapunov function. In addition to ensuring positive recurrence, the drift equation~\eqref{eq:drift} gives a uniform bound on the
average cost of any policy.

It turns out that we also need policy independent bounds on the value function, which is ensured by our next assumption.
\begin{assumption}
\label{asspn:assump3}
We assume that there exist constants $T_B$ and $p_B$, independent of policy $\pi$, such that 
        \begin{equation}
        \label{eq:assump3}
            \P_\pi^{T_B}\pbr{\vqp|\vq} \geq p_B \qquad \qquad \forall \vq\in B, \forall \vqp\in B, \forall \pi\in\Pi,
        \end{equation}
        where $\P_\pi^{T_B}$ is the $T_B$-step probability transition matrix. 
\end{assumption}

Equation~\eqref{eq:assump3} requires that any state $\vq\in B$ can be reached from any state $\vqp\in B$  in at most $T_B$ transitions with atleast $p_B$ probability under any policy $\pi\in\Pi$. 
Equation~\ref{eq:bdddiff} states that it is not possible to move arbitrarily far away from the current state in a single transition under any policy. 


\section{Main Result and Discussion}
We now present the main result, which is the performance of NPG in the context of infinite state MDPs within the learning framework. We then contextualize Assumptions~\ref{asspn:asump1}, \ref{asspn:asump2} and \ref{asspn:assump3} and elaborate on how they can be satisfied in the context of queuing systems.
\subsection{Main Result}
\begin{restatable}{theorem}{maintheorem}
\label{thm:maintheorem}
Consider the sequence of policies $\pi_1,\pi_2,\ldots,\pi_T$ obtained from Algorithm~\ref{alg: NPG} with a state-dependent step size $\eta_\vq = \sqrt{\frac{8\log|\A|}{T}}\frac{1}{M_\vq}$, where $M_\vq = \pbr{2\delta(\vq) + \frac{2}{\epsilon}f^2(\vq)+\frac{4D}{\epsilon}+\overline{c}(\vq)+g_1}$ and $\delta(\vq):= \sup_{\pi\in\Pi}\left\|\wQ_{\pi_k}(\vq,a)-Q_{\pi_k}(\vq,a)\right\|_\infty$. Let $J_{\pi_k}$ be the average cost associated with policy $\pi_k$ and let $J_*$ be the minimum average cost across policy class $\Pi$. Let the learning error satisfy the following: 
    \begin{equation}
    \label{eq:eval_error}
        \E_\pi\sbr{\delta(\vq)}\leq\kappa(\vq) \qquad \forall \vq\in\S, \pi\in\Pi
    \end{equation}
Then, under Assumptions~\ref{asspn:asump1}, \ref{asspn:asump2} and \ref{asspn:assump3}, there exist constants $c', c''$ not depending on $T$ or $\pi_1,\pi_2,\ldots,\pi_T$ such that: 
    \begin{equation}
    \label{eq:main}
        \sum_{k=1}^T \E\pbr{J_{\pi_k} - J_*} \leq c'\sqrt{T} + c'' T
    \end{equation}
    where $c' = \sqrt{\frac{\log|\A|}{2}}\pbr{2\beta+\beta_1+\beta_2+\frac{g}{\epsilon}+g_1}$, $c''= 2\beta$, $\beta:=\E_{\vq\sim d_{\pi^*}}\sbr{\kappa(\vq)}$, $\beta_1 = \frac{4D}{\epsilon}\E_{\vq\sim d_{\pi^*}}\sbr{f(\vq)}$, $\beta_2 = \frac{2}{\epsilon}\E_{\vq\sim d_{\pi^*}}\sbr{f^2(\vq)}$ and $g_1 = \frac{2D^2}{\epsilon}+ (K+C_B+\frac{g}{\epsilon})\pbr{\frac{T_B}{p^2_B}}$.
\end{restatable}
\begin{proof}The proof is in Appendix~\ref{appendx:maintheorem}. An outline is provided in Section~\ref{section:proofoutline}.
\end{proof}

{\subsection{Discussion on Assumptions}}
\label{disc:assumptions}
In this section we discuss how the assumptions can be satisfied in the context of stochastic networks, a broad class of applications. We focus on three main categories of these applications: (i) large but finite state spaces, (ii) countable state spaces with abandonments, and (iii) scheduling in switches.

\subsubsection{Finite but large state spaces}

Consider MDPs with finite state and action spaces, where the state \(\vq \in \{0, 1, \ldots, S\}^K\) is a vector of length \(K\), with each element representing the number of jobs in the corresponding queue. Here, \(S\) denotes the buffer size, making it a finite-state problem. The instantaneous costs are frequently modeled as linear in \(\|\vq\|\), so both \(\overline{c}(\vq) = O(\|\vq\|)\) and \(\underline{c}(\vq) = O(\|\vq\|)\), given that the number of actions is finite. In these applications, choosing \(f(\vq) = \|\vq\|_1\) automatically satisfies \Cref{asspn:asump2}. Due to finiteness of the state space, choosing a sufficiently large $g$ ensures that \Cref{eq:drift} is trivially satisfied.

If the policy and transition kernels ensure a non-zero probability of no job arrivals and no departures across the policy class (which is typical in most queuing systems), it is possible to transition from any state to \(\vq^*\) (which corresponds to the zero state \(\mathbf{0}\), representing empty queues) and from \(\vq^*\) to any other state. This guarantees irreducibility as per \Cref{asspn:asump1} and satisfies \Cref{asspn:assump3}. Since we are working with finite Markov chains, irreducibility is sufficient to ensure the existence of a unique stationary distribution. 

In previous literature that utilized a state-independent fixed step size \( \eta \), the theoretical bound for \( \eta \) is:
\begin{equation}
\label{eq:fixedss}
    \eta \propto \min_{\substack{\pi\in\Pi\\ \vq\in\S, a,a'\in\A}}\frac{1}{|\widehat{Q}_{\pi}(\vq,a)-\widehat{Q}_{\pi}(\vq,a')|}
\end{equation}
It is practically not possible to estimate the value $\eta$ from \Cref{eq:fixedss}. Hence, a broad hypermeter tuning without any guideline was necessary. For instance, consider the case of perfect policy evaluation, that is $Q_\pi(\vq,a) = \widehat{Q}_\pi(\vq,a)$. Then \Cref{eq:fixedss} suggests,
\begin{equation}
    \eta \propto \min_{\substack{\pi\in\Pi\\ \vq\in\S, a,a'\in\A}}\frac{1}{|{Q}_{\pi}(\vq,a)-{Q}_{\pi}(\vq,a')|}
\end{equation}
Our theory suggests (and later experimentally verified) that $Q_\pi(\vq,a)=O\pbr{\|\vq\|^2}$ when $c(\vq)=O\pbr{\|\vq\|}$. Hence this leads to $\eta\propto \frac{1}{\|\vq_{\text{max}}\|^2}$. Hence as the size of the state space increases, the value of $\eta$ reduces, vastly increasing the iteration complexity of NPG as the size of the state space increases. 

In prior literature \cite{abbasi2019politex,murthy2023performance} the assumption over policy evaluation error is a high probability bound as below:
\begin{equation}
\label{eq:lerror_old}
    \sup_{\pi\in\Pi}\left\|Q_\pi - \widehat{Q}_\pi\right\|_{d_\pi} \leq \epsilon
\end{equation}
where $\epsilon>0$ is a constant. From this, a very loose upper bound on $\sup_{\pi\in\Pi}\left\|Q_\pi(\vq,\cdot)-\widehat{Q}_\pi(\vq,\cdot)\right\|_\infty$ is obtained by assuming a lower bound on $d_\pi(q)$ and in theory, the stepsize $\eta$ has to be chosen inversely proportional to this quantity. However, this stepsize can lead to very slow convergence as the size of the state space increases. Potentially, one can search for the best stepsize by treating it as a hyperparamter and tuning it experimentally. However, there are no easy guidelines for this hyperparameter tuning.

On the other hand, the assumption on policy evaluation in \Cref{thm:maintheorem} models the value function estimation error as
\begin{equation}
    \sup_{\pi\in\Pi}\left\|Q_\pi(\vq,\cdot)-\widehat{Q}_\pi(\vq,\cdot)\right\|_\infty \leq \delta(\vq).
\end{equation}
The value function  error consists of two parts: one a function approximation error and another a learning error associated with learning the parameters of the function approximation. First, let us consider the tabular case, i.e., one where the value function is directly estimated for each (state, action) pair without using function approximation. Then, our bounds on the value function indicate an upper bound on $Q$ which is quadratic in state $\|\vq\|$. This bound can thus be leveraged to ensure that $\delta(\vq)\approx\|\vq\|^2$. Moreover, since the approach to obtaining performance bounds of NPG in prior literature \cite{abbasi2019politex,murthy2023performance} does not explicitly characterize upper bounds on the solution to the Poisson's equation, the provable error bounds i.e., $\epsilon$ in \Cref{eq:lerror_old} is thus agnostic of the structure of the state action value function and consequently loose. 

When dealing with learning using linear function approximations, since \Cref{lem:bddQ} indicates $Q_\pi(\vq,a)\leq O\pbr{\|\vq\|^2}$ for all $a\in\A$ (as $f(\vq)=\|\vq\|_1$), choosing a feature space representation $\Phi$, where the largest element of $\phi(\vq,a)$ is quadratic in $\|\vq\|$ yields a learning error $\delta(\vq)\leq O\pbr{\|\vq\|^2}$. Since all moments of $\|\vq\|$ exist, this provides a guideline regarding the choise of a state dependent step size i.e., $\eta(\vq)\propto \frac{1}{\|\vq\|^2}$. The choice of $\eta$ in prior literature explicitly relied on the knowledge of $\sup_{\pi\in\Pi}\max_{\substack{\vq\in\S \\ a\in\A}}\widehat{Q}_\pi(\vq,a)$, hence necessitating a broad hyperparameter search without any prior knowledge to aid with this search.

In the context of really large spaces, a powerful tool employed to approximate $Q$ functions are large scale neural networks, which can be potentially utilized to learn in countable state spaces as well (as elaborated in \Cref{subsubsec:Evaluation}). Since neural networks of sufficient width can approximate continuous functions arbitrarily well, $\delta(\vq)\approx O\pbr{\|\vq\|^2}$
is a reasonable error bound as $Q_\pi(\vq,a)\leq O\pbr{\|\vq\|^2}$ for all $\pi\in\Pi, \vq\in\S, a\in\A.$

\subsubsection{Countable State Spaces with Abandonments}
Abandonments occur when the wait time for service of a job is too long. For instance, in two sided queuing systems, such as those encountered in ride hailing apps such as Uber/Lyft, a person might leave a queue if not serviced sufficiently quickly. Suppose that at each time instant, an individual abandons the queue independently of others with probability \(\nu\). In this case, \Cref{eq:drift} exhibits a strong negative drift as the queue length grows. This can be demonstrated by choosing the Lyapunov function \(f(\vq) := \|\vq\|^2\), similar to the finite state case. The reasoning is that as the queue length increases, the likelihood of abandonments rises accordingly. Thus, in the presence of abandonments and with bounded arrivals and departures within a single time slot, \Cref{asspn:asump2} is naturally satisfied. \Cref{asspn:asump1} and \Cref{asspn:assump3} are satisfied as long as there is a non-zero probability of no job arrivals and no service, in a fashion identical to as described in the finite state setting. 

\subsubsection{Scheduling in Switches}
Switch scheduling is encountered in a wide array of applications such as wireless networks, cloud computing, data centres, etc. Here the state space denotes a vector of job lengths corresponding to different queues. The action space in such a setting corresponds to different possible bipartite matchings from the input queues to the output queues. In such a scenario, the cost associated with a state is independent of the matching chosen and hence can be modeled as $c(\vq):=\|\vq\|_1$. Consequently the Lyapunov function chosen is also identical i.e., $f(\vq)=\|\vq\|_1$. As in most applications, the  number of jobs that can arrive and depart in a single time instant is uniformly bounded. In such systems, the drift \Cref{eq:drift} in \Cref{asspn:asump2} can be satisfied as described below.

In a large class of queueing systems, the MaxWeight policy is known to ensure stability, i.e., positive recurrence (see Chapter 4, \cite{srikant2013communication}). 
    Assumption~\ref{asspn:asump2} is inspired by the so-called MaxWeight policy, which is 
    known to   
    satisfy the drift equation below:
    \begin{equation}
    \label{eq:maxweightdrift}
        \E_{\pi_{\mathsf{MW}}}\sbr{\|\vq_{k+1}\|^2-\|\vq_k\|^2|\vq_k=\vq}\leq -\epsilon\|\vq\|_1 + d_1
    \end{equation}
    where the expectation is taken with respect to $\pi_{\mathsf{MW}}$ and $\epsilon,d_1$ are some positive constants independent of policy. 
    Assumption~\ref{asspn:asump2} is designed so that we explore a family of randomized policies that inherit stability from MaxWeight, while also enabling us to learn policies that outperform MaxWeight. 
    
    In particular, we consider policies obtained by using a combination of MaxWeight and arbitrary randomized acations by transforming the underlying MDP as follows. 
    Let the policies obtained from update Equation~\ref{eq:update} be referred to as $\pi_{\mathsf{NPG}}$. Modify the underlying MDP such that the probability transition kernel corresponds to a policy $\pi$ defined below:
    \begin{equation}
    \label{eq: modifiedMDP}
        \pi(a|\vq) = \begin{cases}
            \pi_{\mathsf{NPG}}(a|\vq) ,&   \text{w.p.} \quad \min\pbr{1,\frac{1}{\lambda \|\vq\|}}\\
            \pi_{\mathsf{MW}}(a|\vq),         &   \text{w.p.} \quad 1 - \min\pbr{1,\frac{1}{\lambda \|\vq\|}}
        \end{cases}
    \end{equation}
    where \(\lambda > 0\) is a fixed parameter with a very small positive value.
    
    As the queue length grows larger, the above transformed MDP enacts the Max-Weight policy with greater probability at higher queue lengths. The value of $\lambda$ decides the threshold at which Max-Weight policy starts influencing the transition dynamics. Once queue lengths exceed $\frac{1}{\lambda}$, this soft thresholding compromises some optimality to prioritize stability. This differs from the hard thresholding approach taken  in \cite{liu2019reinforcement}.
    
    We will now illustrate that this family of soft-thresholded policies satisfy Assumption~\ref{asspn:asump2}. First note that from the bounded arrivals and departures assumption in \Cref{eq:bdddiff}, it is easy to show that $\pi_{\mathsf{NPG}}$ satisfies
    \begin{equation}
        \E_{\pi_{\mathsf{NPG}}}\sbr{\|\vq_{k+1}\|^2-\|\vq_k\|^2|\vq_k=\vq}\leq d_2\|\vq\|_1 + d_3
    \end{equation}
    where the expectation is taken with respect to $\pi_{\mathsf{NPG}}$ and $d_2,d_3$ are some positive constants independent of policy. Thus the drift equation corresponding to policy $\pi$ in Equation~\ref{eq: modifiedMDP} is as follows:
    \begin{align}
        \E_{\pi}& \sbr{\|\vq_{k+1}\|^2-\|\vq_k\|^2|\vq_k=\vq}\leq \nonumber \\
        & \begin{cases}
            \frac{d_2}{\lambda} + d_3 ,&    \|\vq\|_1\leq\frac{1}{\lambda}\\
            \frac{1}{\lambda \|\vq\|_1}\pbr{d_2\|\vq\|_1 + d_3} + \pbr{1-\frac{1}{\lambda \|\vq\|_1}} \pbr{-\epsilon\|\vq\|_1 + d_1},         &    \|\vq\|_1>\frac{1}{\lambda}
        \end{cases}
        \label{eq:mixeddrift}
    \end{align}
    Combining the cases in Equations~\ref{eq:mixeddrift}, we obtain the following drift relation for policy $\pi$ for all $\vq\in\S$:
    \begin{equation}
    \label{eq: final_drift}
        \E_{\pi}\sbr{\|\vq_{k+1}\|^2-\|\vq_k\|^2|\vq_k=\vq}\leq -\epsilon\|\vq\|_1 + D
    \end{equation}
    where $D$ is a constant independent of policy $\pi$ but is a function of constants $d_1, d_2, d_3, \epsilon$ and $\lambda$. Note that the constant $\epsilon$ remains the same in both \eqref{eq:maxweightdrift} and \eqref{eq: final_drift}. This constitutes one such class of policies that satisfies the required the drift equation~\eqref{eq:drift} for our analysis.

{\subsubsection{Policy Evaluation in Stochastic Networks with Countable States}}
\label{subsubsec:Evaluation}

The theory for performing policy evaluation using learning is not well developed for countable state spaces. We present some speculative ideas in this regard, but this requires considerable further work which is beyond the scope of this paper. However, in practice, there has been experimental work for countable state spaces, which seems to indicate learning based control is possible \cite{dai2022queueing}.

It is a well-known fact that neural networks with at least one hidden layer of sufficient width and a non-linear activation function can approximate any continuous function on a compact domain arbitrarily well \cite{cybenko1989approximation,funahashi1989approximate,hornik1989multilayer}. A potential technique to evaluate value functions associated with infinite state spaces can be through neural network temporal difference learning. In order to do so, consider the following transformation to compactify the domain of the problem. Recall that the system comprises of $K$ queues that is, $\vq\in\N^K$. Let $q_i$ represent the number of jobs in the $i^{\text{th}}$ queue. Then define a vector $\mathbf{x}\in\sbr{0,1}^K$ such that the $i^{\text{th}}$ element is $x_i = \frac{1}{1+q_i}$. Given a policy $\pi$, consider a linear function approximation $\wQ_\pi(\vq,a)$ of the state-action value function $Q_\pi(\vq,a)$ as below:
\begin{equation}
    \frac{\wQ_\pi(\vq,a)}{\|\vq\|^2} = {\theta_\pi}^\top \phi\pbr{\mathbf{x(q)},a}
\end{equation}
where the feature vector $\phi$ is defined as below,
\begin{equation}
    \phi\pbr{\mathbf{x(q)},a} = \begin{bmatrix} 
                                \mathbb{I}_{w_1^\top \pbr{\mathbf{x(q)},a} \geq 0} \ \pbr{\mathbf{x(q)},a} \\ \vdots \\ \mathbb{I}_{w_m^\top \pbr{\mathbf{x(q)},a} \geq 0} \ \pbr{\mathbf{x(q)},a}
                                \end{bmatrix}.
\end{equation}
Here, $w_i\sim\mathcal{N}(0,I)$ and $I\in\R^{(K+1)\times(K+1)}$ is the identity matrix. This linearized model is well-studied approximation to a neural network and is called the Neural Tangent kernel (NTK) approximation; see \cite{ji2019polylogarithmic}, for example. We will not discuss the merits of the NTK approximation here since that is irrelevant to our analysis, but we only introduce the NTK approximation to discuss why we chose our model for function approximation. In the NTK approximation, $w_i\in\R^{K+1}$ is random initialization which chooses a random set of features. Each feature vector $\phi\pbr{\mathbf{x(q)},a}$ is of length ${m|\A|K}$, where $m$ represents the width of the hidden layer in the neural network. Finally, $\theta^*_\pi$ represents the optimal parameter vector, i.e., the parameter that best estimates $Q_\pi(\vq,a).$ 

The state action value function $Q_\pi(\vq,a)$ can be approximated arbitrarily well if $Q_\pi$ is a continuous function. This is indeed the case for some simple contexts such as the M/M/1 queue, where the value function is a quadratic function in queue length (\cite{meyn2008control}). More generally, Equation~\eqref{eq:Qbd} indicates that the 
$Q_\pi(\vq,a)$ can be upper bounded by a quadratic function.
Therefore, under the assumption that $Q_\pi$ is continuous, the learning error due to policy evaluation using the neural network can be characterized as follows:
\begin{align}
    \|Q_\pi(\vq,a) - \wQ_\pi(\vq,a)\|&= \|Q_\pi(\vq,a) - {\theta_\pi}^\top \phi\pbr{\mathbf{x(q)},a}{\|\vq\|^2}\|\\
    \label{eq:fapperror}
                                    &\leq\left\|Q_\pi(\vq,a) -{\theta^*_\pi}^\top \phi\pbr{\mathbf{x(q)},a}{\|\vq\|^2}\right\| \\&
    \label{eq:learningerror} 
                                    +\left\|{\theta^*_\pi}^\top \phi\pbr{\mathbf{x(q)},a}{\|\vq\|^2}- {\theta_\pi}^\top \phi\pbr{\mathbf{x(q)},a}{\|\vq\|^2}\right\|
\end{align}
The function approximation error is captured in Equation~\eqref{eq:fapperror} as follows:
\begin{equation}
    \left\|Q_\pi(\vq,a) -{\theta^*_\pi}^\top \phi\pbr{\mathbf{x(q)},a}{\|\vq\|^2}\right\| = \left\|\frac{Q_\pi(\vq,a)}{\|\vq\|^2} -{\theta^*_\pi}^\top \phi\pbr{\mathbf{x(q)},a}\right\|{\|\vq\|^2} \leq \delta_1(m){\|\vq\|^2}
\end{equation}
where $\delta(m)$ is a constant that is independent of the underlying policy but depends on the width of the hidden layer. In fact, it is shown in \cite{satpathi2021dynamics} that when approximating polynomials, as $m\to\infty$, $\delta_1(m)\to 0$. The temporal difference (TD) learning error is captured in Equation~\eqref{eq:learningerror} and is a function of the number of samples available and can be quantified as $\delta_2\|\vq\|^2$ with high probability. And thus, with high probability, the overall state dependent error can be quantified as follows:
\begin{equation}
\label{eq:pe_error}
    \|Q_\pi(\vq,a) - \wQ_\pi(\vq,a)\| \leq \delta \|\vq\|^2
\end{equation}
where $\delta=\delta_1(m)+\delta_2$.

\section{Proof outline and Key Insights}
\label{section:proofoutline}
The difference in average cost associated with a policy \(\pi\) and the optimal average cost is linked to the \(Q_\pi\) function through the performance difference lemma (\cite{cao1999single}) as below:
\begin{equation}
\label{eq:pdl}
    J_\pi -J^*  =  \E_{\vq\sim d_{\pi^*}}\sbr{Q_\pi\pbr{\vq,\pi(\vq)}-Q_\pi\pbr{\vq,\pi^*(\vq)}}.
\end{equation}
Hence, the regret in LHS of Equation~\eqref{eq:main} can be captured in terms of difference in the state action value function \(Q_\pi\). However, in practise it is not possible to determine \(Q_\pi\) exactly since the model might be unknown or the state space is infinite. Hence, we incorporate the estimates \(\wQ_\pi\) of the value function \(Q_\pi\). If the estimates satisfy Equation~\eqref{eq:eval_error}, then from Equation~\eqref{eq:pdl} we obtain the following regret formulation:
\begin{equation}
\label{eq:proofoutline}
    \sum_{k=1}^T \E\pbr{J_{\pi_k} - J^*} \leq 2T \E_{\vq\sim d^{\pi^*}}\kappa\pbr{\vq} + \underbrace{\E_{\vq\sim d_{\pi^*}}\sbr{\E\pbr{\sum_{k=1}^T\wQ_{\pi_k}\pbr{\vq,\pi^*(\vq)}-\wQ_{\pi_k}\pbr{\vq,\pi_k(\vq)}}}}_{(a)}
\end{equation}
The term linear in $T$, i.e., $2\E_{\vq\sim d^{\pi^*}}\kappa\pbr{\vq}$ is a consequence of function approximation and is generally unavoidable \cite{abbasi2019politex}.The primary task is to bound \((a)\) in Equation~\eqref{eq:proofoutline}. We approach this in four steps: (i) examining the link between NPG and prediction through expert advice as highlighted in prior literature, and identifying challenges specific to our countable state-space model and cost structure, (ii) deriving policy-independent bounds on the value functions, i.e., the solution to Poisson's Equation~\eqref{poissons}, (iii) accounting for policy evaluation errors and establishing policy-independent bounds on the estimates of the value function, and (iv) integrating all these steps to achieve the final result. We now proceed with the proof outline. 

\textbf{Step 1 (Connection to Weighted Averaging):} This step involves connecting learning within Markov Decision Processes (MDPs) to prediction through expert advice. This connection was initially identified in \cite{even2009online} for MDPs and later extended to the learning setting in \cite{abbasi2019politex}. We now discuss this connection in some detail and explain why we need our proof techniques to adapt this connection to the countable state-space setting. In the framework of prediction through expert advice, the agent selects an action \(a_t\) at time \(t\), and the environment responds with a corresponding loss \(l_t(a_t)\). Concurrently, an expert follows a predetermined strategy, which in our context can be simplified to a single action \(a^*\) taken at each time step, also experiencing a loss of \(l_t(a^*)\). The agent's objective is to minimize the overall loss by considering all it's past observations when choosing an action. If the expert opts for a fixed strategy \(\pi^*\) over the available actions, the following holds true.
\begin{theorem}
\label{th:weightavg}
    (\textit{Section 4.2, Corollary 4.2, \cite{cesa2006prediction}}.) Consider the exponentially weighted average forecaster problem. Let the set of actions possible at each time step and each instance be denoted by $\A:=\cbr{1,\ldots,n}$. For a fixed instance $s$, let $l_t(s,i)$ be the loss associated with action $i\in\A$ at time $t$ such that for any pair of actions $i,i'\in\A$, 
    \begin{equation}
    \label{eq:boundedloss}
        \left|l_t(s,i)-l_t(s,i')\right|\leq M(s)
    \end{equation}
    Consider the action strategy below:
    \begin{equation}
        \pi_t(i|s) = \frac{\pi_{t-1}(i|s)\exp{\pbr{-\eta_s l_{t-1}(s,i)}}}{\sum_{k=1}^n \pi_{t-1}(k|s)\exp{\pbr{-\eta_s l_{t-1}(s,k)}}}
    \end{equation}
    Then, for any fixed policy $\pi^*$, setting $\eta_s = \sqrt{\frac{8\log n}{T}}\frac{1}{M(s)}$ yields the following overall regret corresponding to instance $s$.
    \begin{equation}
        \sum_{k=1}^T \pbr{l_k\pbr{s,\pi_k(s)} - l_k\pbr{s,\pi^*(s)}} \leq M(s)\sqrt{\frac{T\log n}{2}}
    \end{equation}
\end{theorem}
The NPG algorithm can be interpreted as applying the weighted averaging algorithm to each state \(\vq\) in the state space, with the goal of learning the optimal policy for each state. In this context, the loss function associated with an action \(a\) in state \(\vq\) at time \(k\) is the estimate \(\widehat{Q}_{\pi_k}(\vq, a)\) of the state-action value function, where the policy in use at time \(k\) is \(\pi_k\). However, as indicated by Equation~\eqref{eq:boundedloss}, the loss function—\(\widehat{Q}_{\pi_k}(\vq, a)\)—must be bounded for any given state \(\vq\). In finite-dimensional MDPs, a state-independent uniform bound on the state-action value function is typically assumed \cite{abbasi2019politex}. This is due to the fact that the step-size $\eta$ is assumed to be independent of $s.$ Note that, compared to \cite{even2009online,abbasi2019politex}, we have made a small, but critical, change to the best-experts algorithm by allowing the step-size $\eta$ to be a function of $s.$ When the state-space is countable, the state-action value function \(Q_\pi\) cannot be uniformly bounded and hence, a constant step-size cannot be assumed. With the introduction of a state-dependent step-size, we can choose a different step-size for each state using bounds on the solution to Poisson's equation, i.e., \(Q_\pi(s,a)\), which depends on the state, but is uniform over all policies. Obtaining such bounds is one of the key contributions of the paper.

\textbf{Step 2 (Value Function Bounds):} To establish bounds on Poisson's Equation~\ref{eq:Qfunc}, we initially rely on Assumptions~\ref{asspn:asump1} and \ref{asspn:asump2}. In dealing with countable state space MDPs, along with irreducibility, we require the Markov chain to be positive recurrent for a unique stationary distribution to exist. The drift equation~\ref{eq:drift} along with the rest of Assumption~\ref{asspn:asump2} ensures the positive recurrence of the underlying Markov chain. Since \(Q_\pi\) is related to the state value function \(V_\pi\) (see Equation~\ref{eq:Qfunc}), we initially constrain \(V_\pi\) using Assumptions~\ref{asspn:asump1} and \ref{asspn:asump2}. This leads to an upper bound on \(V_\pi(\vq)\) for all \(\vq \in\S \), 
\begin{equation}
        V_\pi(\vq) \leq \frac{2}{\epsilon}f^2(\vq) + \underbrace{O{\pbr{\E_\pi\sbr{ \sum_{k=0}^{\tau^\pi_{\vq^*}}\mathbb{I}\pbr{\vq_k\in B}\Big|\vq_0=\vq}}}}_{(b)},
\end{equation}
where \(B\) is defined in Equation~\ref{eq: set_B}.
Recall Equation~\eqref{eq:boundedloss} in the context of weighted expert averaging. The constraint on the loss function's bound $\pbr{M(s)}$ must be independent of time. When applied to the NPG framework, this implies the necessity of a policy-independent upper bound on the state-action function \(Q_\pi\), which, in turn, necessitates a policy-independent bound on the state value function \(V_\pi\). For \((b)\) to be well-defined, the drift alone is insufficient, as indicated in previous studies \cite{glynn2024solution,glynn1996liapounov}. Addressing this is the second challenge in our analysis, which we navigate by introducing a mild structural Assumption~\ref{asspn:assump3} commonly satisfied in stochastic networks. 

These structural assumptions yield a uniform upper bound on the hitting time of state \(\vq^*\), defined in \Cref{def:qstar}, when starting from any point within \(B\). This uniform upper bound on hitting time aids in bounding the state value function \(V_\pi\) from below.  The drift inequality \eqref{eq:drift} along with a bound on hitting time assists in bounding the value function \(V_\pi\) from above. As a consequence, we obtain the following, 
\begin{equation}
\label{eq:Qbd}
    \left|Q_\pi(\vq,a)-Q_\pi(\vq,a')\right|\leq O(f^2(\vq)) \qquad \forall \pi\in\Pi, \forall a,a'\in\A \ \text{and} \ \forall \vq\in\S
\end{equation}
As a result, we establish policy-independent bounds on the value function \(Q_\pi\). While the drift assumption~\ref{asspn:asump2} played a crucial role in deriving policy-dependent bounds on the value function \(V_\pi\), for the purpose of NPG, we need these bounds to be independent of the policy. The structural assumption~\ref{asspn:assump3} eliminates this policy dependence. Consequently, from Equation~\ref{eq:Qfunc}, this translates into policy-independent bounds on \(Q_\pi\).

\textbf{Step 3 (Handling Estimation Errors):} Since our loss function in the context of Theorem~\ref{th:weightavg} is \(\wQ_\pi\), we need uniform bounds on \(\wQ_\pi\). We leverage the bounds on $Q_\pi$ obtained in Equation~\ref{eq:Qbd} and in conjunction with the evaluation error as modeled in \Cref{thm:maintheorem}, we obtain the following:
\begin{equation}
\label{eq:estimate_error}
    \left|\wQ_\pi(\vq, a) - \wQ_\pi(\vq, a')\right| \leq O(f^2(\vq))+O(\delta(\vq)) \qquad \forall \pi\in\Pi, \forall a,a'\in\A \ \text{and} \ \forall \vq\in\S
\end{equation}
Adapting Equation~\ref{eq:boundedloss} to the context of context of infinite state NPG, implies that $M_\vq = O(f^2(\vq))+O(\delta(\vq))$.



\textbf{Step 4 (Piecing it all together):} 
The upper bound $M_\vq$ on $\wQ_\pi$ in Step 3 is utilized to determine the state dependent step size as $\eta_\vq = \sqrt{\frac{8\log|\A|}{T}}\frac{1}{M_\vq}$. With bounds over $\wQ$ quantified in Equation~\ref{eq:estimate_error}, $(a)$ of Equation~\eqref{eq:proofoutline} is upper bounded by leveraging the connection to the prediction through expert advice Theorem~\ref{th:weightavg}. This yields the final result. 

The detailed proof of all steps and the main theorem can be found in Appendix.

\section{Simulations}
In this section, we empirically evaluate the performance of the algorithmic change proposed in the convergence of natural policy gradient. We consider tabular policies and finite state spaces. Motivated by autoscaling in cloud computing, we consider the following two settings. 

\subsection{Single Queue System}
\subsubsection{Setting}
\label{sec:single_queue}
We consider a single queue system of finite buffer size $B$. Jobs arrive as a Poisson process with rate $\Lambda = 0.45$. There are two service rates $\mu_1=0.5$ and $\mu_2=0.8$ available, where time taken to service a job under $\mu_i$ is distributed as $\text{Exp}(\mu_i)$.  The state space of this system corresponds to the number of jobs $q$ in the buffer, waiting to be serviced. The action $a$ at each state is the choice of the service rate. Hence $|\S|=B+1$ and $\A=2$. The policy is a probability vector over these two actions corresponding to each state. The discrete time probability transition matrix is the one obtained through uniformization of the CTMC corresponding to this system. The instantaneous cost at state $q$ under action $a=\mu_i$ is $c(q,a)=q+c_i$, where $c_1=1$ and $c_2=10$. Utilizing a higher service rate incurs a higher cost but ensures faster job completion, thereby reducing the overall queue length. 

\subsubsection{Policy Evaluation}
We use the $\text{TD}(\lambda)$ algorithm to evaluate the state-action value function \( Q_\pi \) for each policy. For further details on the algorithm, please refer to \cite{zhang2021finite}. First, we generate a trajectory of length \( n \) according to the transition kernel described earlier. The average cost is estimated by averaging the instantaneous costs obtained from the trajectory, and this estimate is then used to evaluate the state-action value function \( Q \). In these simulations, the learning rate is set to \( \beta = 0.1 \) and \( \lambda = 0.95 \).

\subsubsection{Policy Improvement}
The policy improvement step is as in \Cref{eq:update}. The initial policy is chosen to be uniform across all actions. 
Our theory on bounding the solution to the Poisson's Equation suggests that $Q(q,a)$ is of the order of $\frac{1}{q^2}$, with constants that may depend on the problem parameters. Therefore, to test the robustness of our algorithm, we choose $\eta_q = \frac{k}{q^2}$, independent of the problem parameters. To the best of our knowledge, there are no guidelines given for how to choose a fixed step size $\eta$ in prior literature. But based on our theory, since $Q_{\text{max}}$ is of the order $B^2$, where $B$ is the buffer size, we chose $\eta=\frac{1}{B^2}$ for fixed step size NPG. Note that previously there was no guideline to even choose a fixed $\eta$ in prior work, but our bounds on the solution to the Poisson's equation can be used to choose a state-independent $\eta$ as analyzed in prior work. 

\subsubsection{Observations}

\begin{figure}[t]
    \centering
    \subfigure[Single queue with buffer size $B$=20]{
    \includegraphics[width=0.46\linewidth]{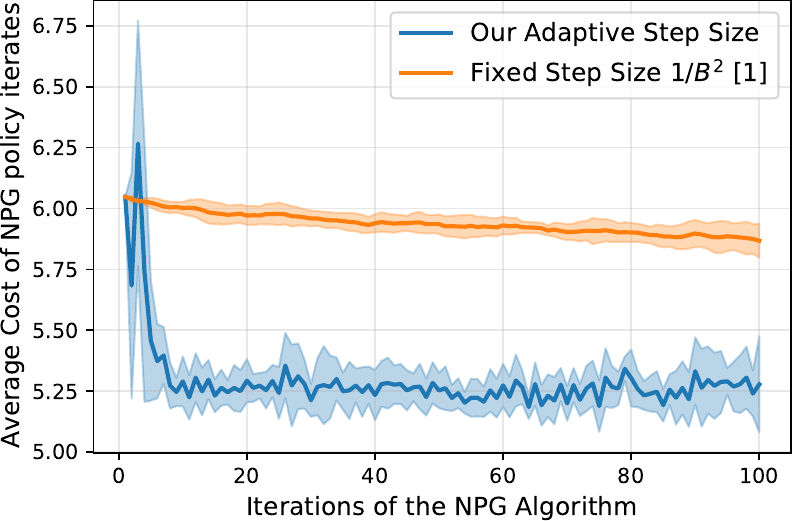}
    \label{fig: b20}
    }
    \hfill
    \subfigure[Single queue with buffer size $B$=50]{
    \includegraphics[width=0.46\linewidth]{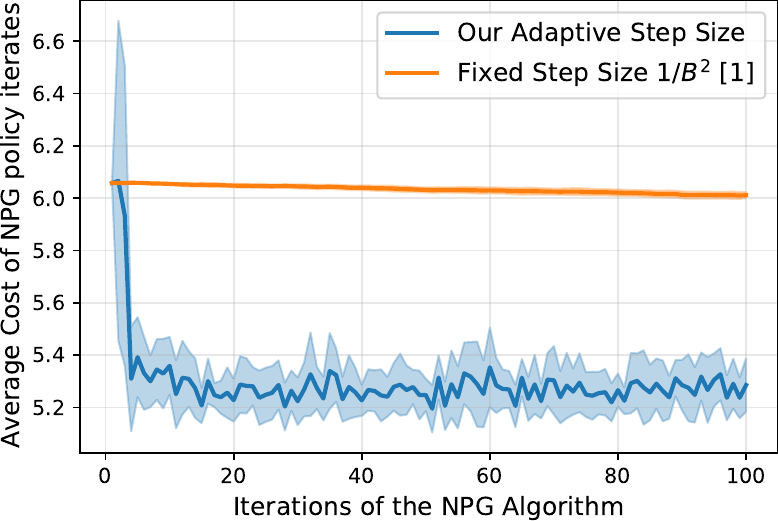}
    \label{fig: b50}
    }
    \vfill
    \subfigure[Single queue with buffer size $B$=75]{
    \includegraphics[width=0.46\linewidth]{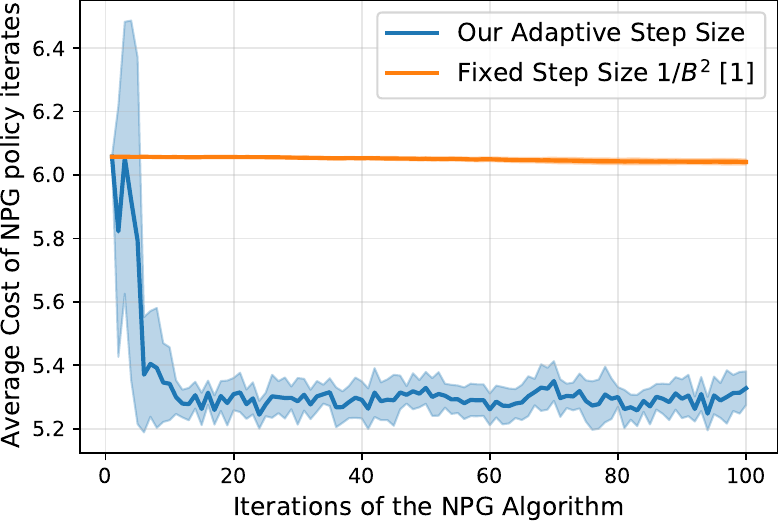}
    \label{fig: b75}
    }\hfill
    \subfigure[Single queue with buffer size $B$=100]{
    \includegraphics[width=0.46\linewidth]{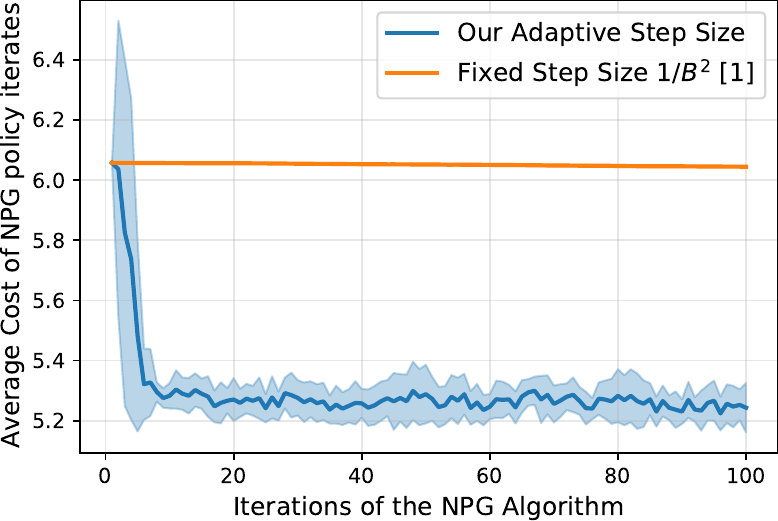}
    \label{fig: b100}
    }
    \caption{NPG in a single queue system}
    \label{fig:single_queue}
\end{figure}

The simulations for this setting are depicted in \Cref{fig:single_queue}. The blue lines represent the performance of NPG with an adaptive step size, while the orange lines indicate the performance with a fixed step size. The y-axis represents the average cost of the policies gnerated through the two NPG algorithms. The x-axis represents the iteration number. \Cref{fig: b20} corresponds to a queueing system with a single queue with a maximum buffer size of 20 (jobs that arrive after the buffer is full are dropped). The length of the trajectory for policy evaluation i.e., $n=3000.$ The performance is averaged over 15 runs of both algorithms. \Cref{fig: b50} corresponds to a buffer capacity of 50 jobs, with $n=5000$, averaged over 15 runs. The step size is set as $\eta_\vq = \frac{1}{q^2}$ for both these instances. \Cref{fig: b75} corresponds to a buffer capacity of 75 jobs, with $n=8000$, averaged over 10 runs. Similarly, \Cref{fig: b75} corresponds to a buffer capacity of 100 jobs, with $n=10000$, averaged over 15 runs. For the latter two cases $\eta_q=\frac{0.5}{q^2}$.

\subsection{Two Queue System}
\begin{figure}[t]
    \centering
    \subfigure[Two queue with buffer size $B$=5]{
    \includegraphics[width=0.46\linewidth]{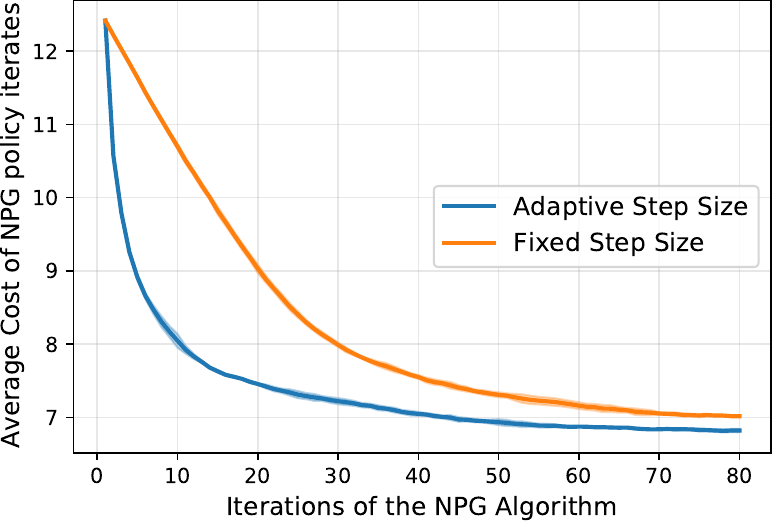}
    \label{fig:JSQ5}
    }
    \hfill
    \subfigure[Two queue with buffer size $B$=10]{
    \includegraphics[width=0.46\linewidth]{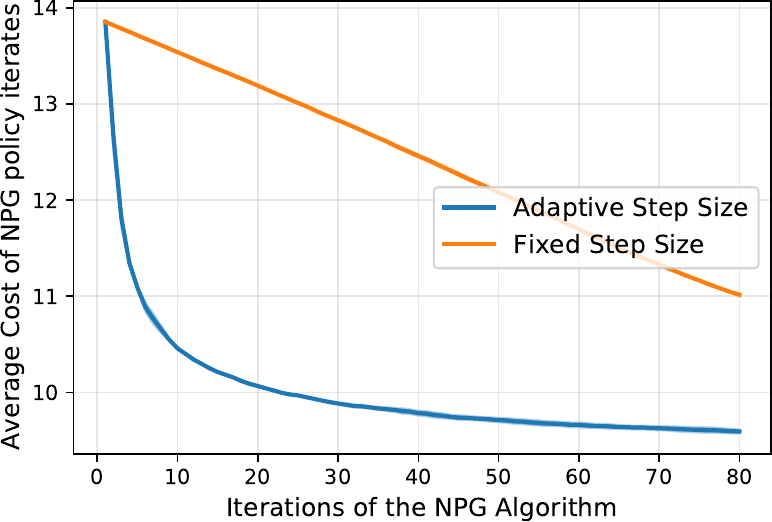}
    \label{fig:JSQ10}
    }
    \vfill
    \subfigure[Two queue with buffer size $B$=15]{
    \includegraphics[width=0.46\linewidth]{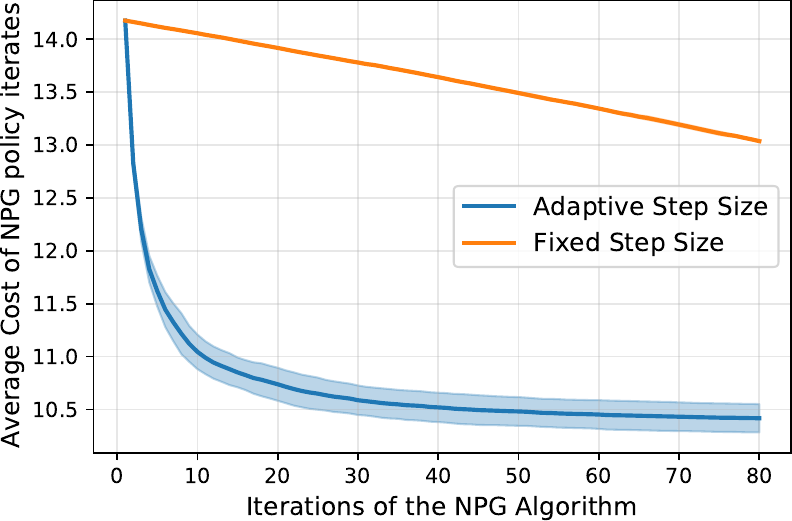}
    \label{fig:JSQ15}
    }\hfill
    \subfigure[Two queue with buffer size $B$=25]{
    \includegraphics[width=0.46\linewidth]{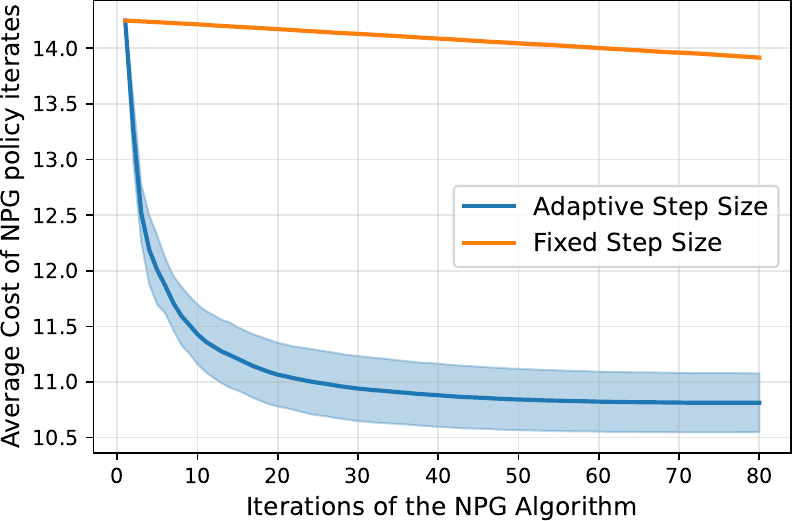}
    \label{fig:JSQ25}
    }
    \caption{NPG in a two queue system}
    \label{fig:double_queue}
\end{figure}

\subsubsection{Setting}
We consider a system with two queues each with buffer size $B$. Jobs arrive as a Poisson process at rate $\Lambda=0.45$ and are routed according to the JSQ (join the shortest queue) policy. Each queue has two service rate options $\mu_1=0.25$ and $\mu_2=0.3$ with $c_i$ as in Setting 1 in \Cref{sec:single_queue}. The state of the system is now a vector $\vq=(q_1,q_2)$ representing the number of jobs in both queues. The action is the choice of service rates for both queues. The cost when employing $a=(\mu_i,\mu_j)$ in state $\vq$ is $c(\vq,a) = q_1 + q_2 +c_i + c_j$. Higher service rate incurs a higher cost but ensures faster job completion. 

\subsubsection{Policy Evaluation}
We use $TD(\lambda)$ algorithm for average cost MDPs as in the previous setting. We first generate a trajectory of length $n$, estimate the average cost and use this estimate to learn the $Q$ function. We set the learning rate $\beta=0.1$ and $\lambda=0.1$.

\subsubsection{Policy Improvement}
We compare the NPG algorithm with two different learning rates $\eta$ namely the adaptive stepsize and the fixed step size. The policy improvement is as in \Cref{eq:update} with $\eta_\vq=\frac{k}{(q_1+q_2)^2}$, which is chosen based on our theory and the fixed step size is set as $\eta=\frac{1}{Q_{\text{max}}}$. Since $Q_{\text{max}}$ is of the order $4B^2$ where $B$ is the buffer capacity for both queues, the fixed step size is thus chosen to be $\frac{1}{4B^2}$.

\subsubsection{Observations}

\Cref{fig:JSQ5}, \Cref{fig:JSQ10}, \Cref{fig:JSQ15} and \Cref{fig:JSQ25} corresponds to a buffer capacity of $(5,5), (10,10), (15,15)$ and  $(25,25)$ jobs respectively. The length of the trajectory for policy evaluation for all four settings is $n=1000$. The step size for the first two cases is $\eta_\vq=\frac{1}{(q_1+q_2)^2}$ whereas for the latter two it is $\eta_\vq=\frac{0.5}{(q_1+q_2)^2}$. The performance is averaged over 3 runs for the first three cases and over 5 runs for the last case. 

\subsection{Noiseless Setting}
We also examine the case with no learning error, i.e., exact evaluation, to determine convergence rates for both step sizes in the previously described settings. Due to the NPG policy improvement update, the sequence of average costs is monotonic. In the single-queue scenario, where the optimal average cost is approximately 4.89 across all buffer sizes, we plot the number of iterations needed to for the cost to fall below 4.96. In \Cref{fig:SQ_Exact}, the y-axis shows the number of iterations required to reach this cost threshold. In the case of two queues, with a buffer size of 5, the optimal average cost is approximately 6, increasing to 10.17 when the buffer size grows to 25. In the plots, we compare the performance of our algorithm to the fixed step size algorithm by comparing the number of iterations needed for the average cost to fall below a threshold: for a buffer size of 5, the threshold we choose was 6; for a buffer size of 10, the threshold was 9.2; and for buffer sizes of 15 and 25, the threshold was set to be 10.26. 
\begin{figure}[t]
    \centering
    \subfigure[Single queue with buffer size $B$=5]{
    \includegraphics[width=0.46\linewidth]{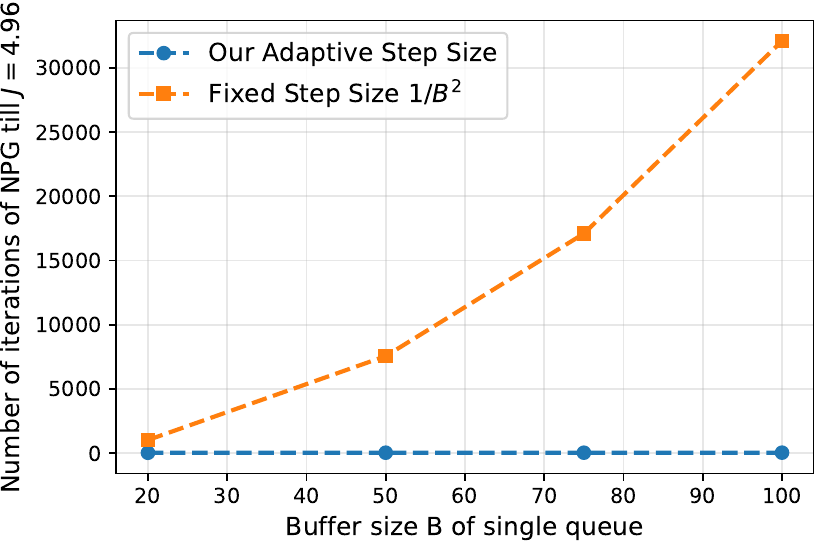}
    \label{fig:SQ_Exact}
    }
    \hfill
    \subfigure[Double queue with buffer size $B$=10]{
    \includegraphics[width=0.46\linewidth]{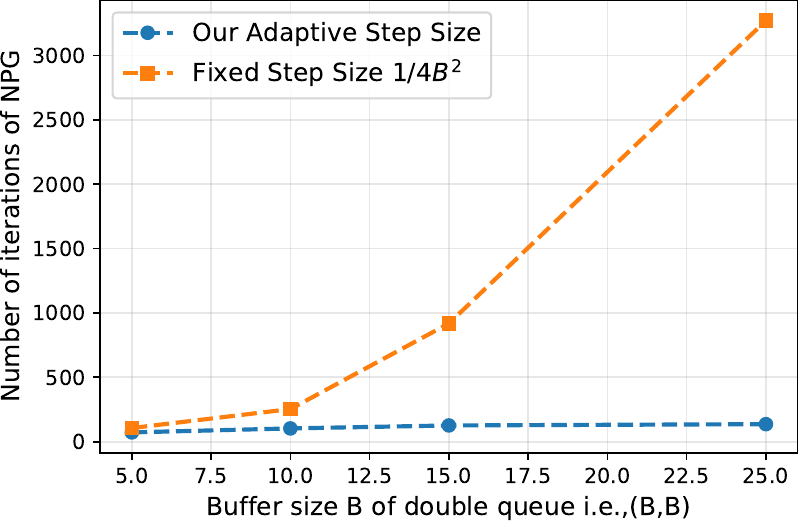}
    \label{fig:JSQ_Exact}
    }
    \caption{NPG with perfect policy evaluation}
\end{figure}
In the case of single queues, using our adaptive step size \(\eta_q = \frac{1}{q}\), the cutoff criterion is reached in roughly 35 iterations, regardless of buffer size. In contrast, with a fixed step size, the number of iterations required to meet the cutoff criterion increases by orders of magnitude as the buffer size grows, as shown in \Cref{fig:SQ_Exact}. 
Similarly, for two queues, our adaptive step size \(\eta_\vq = \frac{1}{q_1+q_2}\) achieves the cutoff criterion in approximately 100 iterations, independent of the state space size. However, when using a fixed step size, the number of iterations required to reach the cutoff criterion follows the pattern illustrated in \Cref{fig:JSQ_Exact}.

\subsection{Key Takeaways}
\begin{itemize}
    \item The NPG algorithm with the adaptive learning rate seems to converge to the near optimal policy in a state-space cardinality independent manner. The magnitude of the slope of the orange line in \Cref{fig:single_queue} and \Cref{fig:double_queue} reduces as the buffer size increases indicating larger number of NPG iterations as the state space grows where as an adaptive step size doesn't face this issue, thus confirming the observations from our theoretical analysis.
    \item The number of iterations to converge to near optimal policy is more or less similar in the context of perfect information and with learning. This suggests that the algorithm is robust to greater errors in the value function estimates of states not visited frequently enough. Hence, the proposed learning rate accommodates realistic learning errors.
    \item Previous literature lacked a heuristic for selecting an effective step size for the NPG algorithm. In contrast, our analysis, based on bounds from Poisson's Equation, offers a state-dependent rule of thumb that significantly improves upon prior step size choices and requires minimal knowledge of the specific MDP instance.
\end{itemize}

\bibliographystyle{ieeetr}
\bibliography{neurips.bib}





\newpage
\appendix

\section{Appendix / supplemental material}

The proof of Step 1 can be found in Chapter 4 of \cite{cesa2006prediction}.

\subsection{Proof of Step 2}

\label{appdx:avgcost}
The following lemmas are a consequence of Assumptions~\ref{asspn:asump1} and \ref{asspn:asump2}. 

\begin{lemma}
    \label{lem: positive_recurrence}
    Let $\P_\pi$ be an irreducible transition matrix on the countable state space $\S.$ Suppose that \eqref{eq:drift} is satisfied. Then the corresponding homogenous Markov Chain is positive recurrent. Consequently, the stationary distribution $d_\pi$ corresponding to $\P_\pi$ exists and is unique \cite{bremaud2013markov}.
\end{lemma}

\begin{lemma}
\label{lem:hajek}
\textcolor{black}{
    Suppose Assumptions~\ref{asspn:asump1} and \ref{asspn:asump2} hold. Let $\vq_{ss}$ be a random variable on $\S$ distributed according to $d_\pi.$ Then, there exists a positive constant $\alpha$ such that $\mathbb{E}_{d_\pi}[e^{\alpha f(\vq_{ss})}] < \infty$.  
    \cite{hajek1982hitting,eryilmaz2012asymptotically}.}
\end{lemma}

This lemma ensures that for all policies $\pi\in\Pi$, all moments of $f(\vq)$ exist. The {first and second} moments are particularly important since final regret depends on them.

As in \cite{eryilmaz2012asymptotically}, Lemmas~\ref{lem: positive_recurrence} and \ref{lem:hajek} can be used to establish the following policy {independent} upper bound on the infinite-horizon average-cost. We also present a proof for completeness.

\begin{lemma}
    \label{lem:avgcost}
    Given Assumptions~\ref{asspn:asump1} and \ref{asspn:asump2}, for all policies $\pi\in\Pi$ it is true that, 
    \begin{equation}
        J_\pi \leq \frac{g}{\epsilon}
    \end{equation}
    where $J_\pi=\E_{d_\pi}\sbr{c_\pi(\vq)}$ is the average cost associated with policy $\pi$ and constants $g,\epsilon$ are the drift parameters in Equation~\ref{eq:drift}.
\end{lemma}

\begin{proof}
 From Assumption~\ref{asspn:asump2}, it follows that for any policy $\pi\in\Pi$, the following drift inequality is satisfied $\forall \vq\in\S$,
    \begin{equation}
    \label{eq:drift_copied}
        \E_\pi\sbr{{f^2(\vq_{k+1})}-{f^2(\vq_{k})} |\vq_k=\vq} \leq -\epsilon \overline{c}(\vq) + g.
    \end{equation} 
    
    Recall that $d_\pi$ represents the stationary measure associated with policy $\pi$. Given that we assume all policies induce irreducible Markov chains and, based on Lemma~\ref{lem: positive_recurrence}, the drift equation~\eqref{eq:drift} ensures the Markov chain's positive recurrence, the stationary distribution \(d_\pi\) exists and is unique. Since $d_\pi\geq 0$, consider the following weighted drift inequality:
    \begin{equation}
    \label{eq:weighted_drift}
        \sum_{\vq\in\S} d_\pi(\vq)\pbr{\E_\pi\sbr{{f^2(\vq_{k+1})}-{f^2(\vq_{k})} |\vq_k=\vq}} \leq -\epsilon\sum_{\vq\in\S} d_\pi(\vq)\overline{c}(\vq) + g
    \end{equation}
    From Lemma~\ref{lem:hajek}, recall that the {second} moment of $f(\vq)$ is defined and exists for all policies $\pi\in\Pi.$ Hence the left hand summation in Equation~\ref{eq:weighted_drift} is well defined. Since the expectation is taken with respect to $\P_\pi$ and since $d_\pi \P_\pi=d_\pi$, the left hand summation in Equation~\ref{eq:weighted_drift} is 0. Hence the expected drift in stationarity is zero.
    We thus obtain the following:
    \begin{align*}
        \epsilon\sum_{\vq\in\S} d_\pi(\vq)\overline{c}(\vq) \leq g
    \end{align*}
    From definition of $\overline{c}(\vq)$, it follows that 
    \begin{align*}
        \epsilon\sum_{\vq\in\S} d_\pi(\vq)c_\pi(\vq) \leq g
    \end{align*}
    Since $J_\pi = \sum_{\vq\in\S} d_\pi(\vq)c_\pi(\vq)$, we thus obtain, 
    \begin{align}
         J_\pi \leq \frac{g}{\epsilon}
        \label{eq:Jpi}
    \end{align}
    Equation~\ref{eq:Jpi} is true for all policies $\pi$. Hence, the average cost is uniformly upper bounded by $\frac{g}{\epsilon}.$ 
\end{proof}

Since $Q_\pi$ is related to $V_\pi$ through Equation~\eqref{eq:Qfunc}, in order to bound $Q_\pi$, we first bound $V_\pi$.
We now derive an upper bound on the value function $V_\pi$ utilizing the drift equation \ref{eq:drift} and the uniform upper bound on $J_\pi$ in Equation~\ref{eq:Jpi}. First we leverage Assumptions~\ref{asspn:asump1} and \ref{asspn:asump2} to establish policy dependent upper bounds on the value function as elaborated in the following subsection.

\subsubsection{Policy Dependent Upper Bound on the State Value Function}
\begin{lemma}
    \label{lem:upperbdoverall}
        Consider a set $B$ defined in Equation~\ref{eq: set_B}. Let $V_\pi(\vq)$ represent the state value function associated with state $\vq\in\S$ and policy $\pi\in\Pi$. Under Assumptions~\ref{asspn:asump1} and \ref{asspn:asump2}, for all $\vq\in \S$ and for all policies $\pi\in\Pi$, there exists policy independent constants $K>0$ and $C_B>0$ such that,
        \begin{equation}
         V_\pi(\vq)\leq \frac{2}{\epsilon}f^2(\vq) + (K+C_B)\pbr{\E_\pi\sbr{ \sum_{k=0}^{\tau^\pi_{\vq^*}}\mathbb{I}\pbr{\vq_k\in B}\Big|\vq_0=\vq}},
        \end{equation}
        where $\tau^\pi_{\vq^*}$ is the time to hit a fixed state $\vq^*\in B$ when starting at $\vq.$
    \end{lemma}
\label{appdx:Lemma34}
    \begin{proof}
The key idea behind the proof is to apply \cite[Theorem 2.1]{glynn1996liapounov} to an appropriately defined drift inequality. We note that the theorem cannot be directly applied to \ref{eq:drift} because it does not satisfy the conditions of the theorem.  Define the following set:
    \begin{equation}
        A_\pi := \cbr{\vq\in\S: \overline{c}(\vq) \leq \frac{2g}{\epsilon}-\E_{d_\pi}\sbr{c_\pi(\vq)}}
        \label{eq:Api}
    \end{equation}
    {Since \(J_\pi \leq \frac{g}{\epsilon}\), it follows from~\Cref{eq: set_B} in~\Cref{asspn:asump2} that \(A_\pi\) is a finite, non-empty set.} Multiplying \eqref{eq:drift_copied} throughout by $\frac{2}{\epsilon}$, we obtain the following:
    \begin{equation}
        \E_\pi\sbr{\frac{2}{\epsilon}f^2(\vq_{k+1})-\frac{2}{\epsilon}f^2(\vq_k)\Big|\vq_k=\vq} \leq -2\overline{c}(\vq) + \frac{2g}{\epsilon}
    \end{equation}
    Consider a $\vq\in A_\pi^\mathsf{c}$. Then, from definition it is true that $-\overline{c}(\vq)\leq -\frac{2g}{\epsilon}+\E_{d_\pi}\sbr{c_\pi(\vq)}$. Bounding $-\overline{c}(\vq)$ from above, we obtain, 
    \begin{equation}
         \E_\pi\sbr{\frac{2}{\epsilon}{f^2(\vq_{k+1})}-\frac{2}{\epsilon}{f^2(\vq_{k})}\Big|\vq_k=\vq} \leq -\overline{c}(\vq) + \E_{d_\pi}\sbr{c_\pi(\vq)}
    \end{equation}
    Recall the definition of set $B$ in \Cref{eq: set_B}. Since the instantaneous costs $c(\vq,a)$ are non-negative for all state-action pairs $(\vq,a)$, the average cost $J_\pi$ is also non-negative for all policies $\pi\in\Pi$. Hence, we obtain that $A_\pi\subset B$ for all $\pi\in\Pi$. 

    Since $B^\mathsf{c}\in\A_\pi^\mathsf{c}$, we thus obtain for all $\vq\in B^\mathsf{c}$, it is true that, 
    \begin{equation}
         \E_\pi\sbr{\frac{2}{\epsilon}{f^2(\vq_{k+1})}-\frac{2}{\epsilon}{f^2(\vq_{k})}\Big|\vq_k=\vq} \leq -\overline{c}(\vq) + \E_{d_\pi}\sbr{c_\pi(\vq)}
    \end{equation}
    Since $\overline{c}(\vq)\geq c_\pi(\vq)$ for all $\vq\in\S$,
    \begin{equation}
        \E_\pi\sbr{\frac{2}{\epsilon}{f^2(\vq_{k+1})}-\frac{2}{\epsilon}{f^2(\vq_{k})}\Big|\vq_k=\vq} \leq -c_\pi(\vq) + \E_{d_\pi}\sbr{c_\pi(\vq)}
    \end{equation}
    Recall from our \Cref{asspn:assump3}, the set $B$ is finite and single step transitions are uniformly bounded. Therefore consider the following definition:
    \begin{equation}
        K := \max_{\substack{{\vqp:\P(\vqp|\vq,a)>0} \\ {\vq\in B,a\in\A}}} \frac{2}{\epsilon}f^2(\vqp)
    \end{equation}
    Hence for all $\vq\in\S$, it is true that,
    \begin{equation}
    \label{eq:finaldrift}
        \E_\pi\sbr{\frac{2}{\epsilon}{f^2(\vq_{k+1})}-\frac{2}{\epsilon}{f^2(\vq_{k})}\Big|\vq_k=\vq} \leq \pbr{-c_\pi(\vq) + \E_{d_\pi}\sbr{c_\pi(\vq)}}\mathbb{I}\pbr{\vq\in B^\mathsf{c}} + K\mathbb{I}\pbr{\vq\in B}
    \end{equation}
    Let $\Tilde{c}_\pi(\vq):= c_\pi(\vq) - \E_{d_\pi}\sbr{c_\pi(\vq)}$. From definition of $B$ in \Cref{eq: set_B}, it is true that for all $\vq\in B$, $\underline{c}(\vq)\leq \frac{2g}{\epsilon}$. Hence for all $\vq\in B^\mathsf{c}$, $c_\pi(\vq)>\frac{2g}{\epsilon}$. And since from \Cref{lem:avgcost} it is true that $\E_{d_\pi}\sbr{c_\pi(\vq)} \leq \frac{g}{\epsilon}$, we obtain $\Tilde{c}_\pi(\vq)>0$, for all $\vq\in B^\mathsf{c}$.

Define
    \begin{equation}
    \label{def:qstar}
        \vq^*:=\argmin_{\vq\in B} \underline{c}(\vq),
    \end{equation}
and let $\tau^\pi_{\vq^*}$ be the first time to hit $\vq^*$ under policy $\pi$ starting from some state $\vq_0=\vq$. Now, applying \cite[Theorem 2.1]{glynn1996liapounov}, we get

    \begin{align}
    \E_\pi\Big[\sum_{k=0}^{ \tau^\pi_{\vq^*}-1}\pbr{\Tilde{c}_\pi(\vq_k)}\pbr{1 -   \mathbb{I}\pbr{\vq_k\in B}}\Big|\vq_0=\vq\Big] \leq \frac{2}{\epsilon}f^2(\vq) + K\E_\pi\sbr{ \sum_{k=0}^{\tau^\pi_{\vq^*}-1}\mathbb{I}\pbr{\vq_k\in B}\Big|\vq_0=\vq}
    \end{align}
    \begin{align}
    \E_\pi\Big[\sum_{k=0}^{ \tau^\pi_{\vq^*}-1}\pbr{\Tilde{c}_\pi(\vq_k)}\Big|\vq_0=\vq\Big] \leq \frac{2}{\epsilon}f^2(\vq) & + K\E_\pi\sbr{ \sum_{k=0}^{\tau^\pi_{\vq^*}-1}\mathbb{I}\pbr{\vq_k\in B}\Big|\vq_0=\vq}+ \\ & \E_\pi\sbr{\sum_{k=0}^{ \tau^\pi_{\vq^*}-1}\pbr{\Tilde{c}_\pi(\vq_k)\mathbb{I}\pbr{\vq_k\in B}}\Big|\vq_0=\vq}   
    \end{align}
    Since $J_\pi$ is non negative,
    \begin{equation}
        \Tilde{c}_\pi(\vq_k)\mathbb{I}\pbr{\vq_k\in B} \leq \max_{\vq\in B}\overline{c}(\vq)=:C_B
    \end{equation}
    Thus,
    \begin{align}
        \E_\pi\Big[\sum_{k=0}^{ \tau^\pi_{\vq^*}-1}\pbr{\Tilde{c}_\pi(\vq_k)}\Big|\vq_0=\vq\Big] \leq \frac{2}{\epsilon}f^2(\vq) & + (K+C_B)\pbr{\E_\pi\sbr{ \sum_{k=0}^{\tau^\pi_{\vq^*}-1}\mathbb{I}\pbr{\vq_k\in B}\Big|\vq_0=\vq}}
    \end{align}
    From \Cref{eq:valfunc_rec}, we thus obtain the following bound on the value function for all $\vq\in\S$,
    \begin{align}
        V_\pi(\vq)\leq \frac{2}{\epsilon}f^2(\vq) & + (K+C_B)\pbr{\E_\pi\sbr{ \sum_{k=0}^{\tau^\pi_{\vq^*}-1}\mathbb{I}\pbr{\vq_k\in B}\Big|\vq_0=\vq}}
    \end{align}
    \end{proof}

In order to invoke the connection of NPG to prediction through expert advice, we need policy independent bounds on the estimate $\wQ_\pi.$ As a step towards achieving that, we first need to establish policy independent bounds on the exact value function $Q_\pi$ and therefore on $V_\pi$. Since the drift provides us with a policy dependent upper bound alone, we exploit the structure of queuing systems in order to obtain a policy independent lower bound and upper bound. 

\subsubsection{Policy Independent Bounds on the State Value Function}
\label{appendx:valfunc}
The structural assumption~\ref{asspn:assump3} aids in obtaining policy independent bounds by providing an uniform upper bound on the time spent in $B$ till state $\vq^*$ is reached starting from any state $\vq\in\S$.
\begin{lemma}
\label{lem:bddhittingtime}
Consider the set $B$ in Equation~\eqref{eq: set_B}. Define $\tau^{\text{bound}}_B$ as 
\begin{equation}
    \tau^{\text{bound}}_B = \max_{\substack{\vq\in\S\\\pi\in\Pi}} \E_\pi\sbr{ \sum_{k=0}^{\tau^\pi_{\vq^*}-1}\mathbb{I}\pbr{\vq_k\in B}\Big|\vq_0=\vq}
\end{equation}
Then under Assumption~\ref{asspn:assump3}, for any policy $\pi\in\Pi$, $\tau^{\text{bound}}_B$ satisfies
\begin{equation}
    \tau^{\text{bound}}_B \leq \frac{T_B}{p_B^2}.
\end{equation}
\end{lemma}
\begin{proof} 
Since the $\mathbb{I}\pbr{\vq_k\in B}$ is non-zero only when $\vq_k\in B,$
    \begin{equation}
        \E_\pi\sbr{ \sum_{k=0}^{\tau^\pi_{\vq^*}-1}\mathbb{I}\pbr{\vq_k\in B}\Big|\vq_0=\vq} \leq \max_{\vq\in B} \E_\pi\sbr{ \sum_{k=0}^{\tau^\pi_{\vq^*}-1}\mathbb{I}\pbr{\vq_k\in B}\Big|\vq_0=\vq}
    \end{equation}
   Thus, we can assume $q_0\in B.$  Let \(\tau_n\) denote the time at which the Markov chain \(\vq_k\) enters the set \(B\) for the \(n\)-th time. Let $\widetilde{\vq}_k = \vq_{\tau_k}$. Then, from strong Markov property we know that $\widetilde{\vq}_k$ is also a Markov chain over $B$. Let \(\widetilde{\tau}^\pi_{\vq^*}\) denote the time at which the state \(\vq^* \in B\) is first reached under policy \(\pi\) in the Markov chain \(\widetilde{\vq}_k\). Then,
    \begin{equation}
        \E_\pi\sbr{ \sum_{k=0}^{\tau^\pi_{\vq^*}-1}\mathbb{I}\pbr{\vq_k\in B}\Big|\vq_0=\vq} \leq \E_{{\pi}}\sbr{\widetilde{\tau}^\pi_{\vq^*}|\vq_0=\vq}.
    \end{equation}
    Denoting the transition kernel of $\widetilde{\vq}$ by $\widetilde{\P},$ we have
    \begin{align}
    \E_{{\pi}}\sbr{\widetilde{\tau}^\pi_{\vq^*}} &= \sum_{k=1}^{\infty} k{\widetilde{\P}}_{\pi}\pbr{\widetilde{\tau}^\pi_{\vq^*}=k|\vq_0=\vq} 
    \\ & \leq \sum_{k=1}^{\infty} k T_B \widetilde{\P}_{{\pi}}\pbr{(k-1)T_B< \widetilde{\tau}^\pi_{\vq^*}\leq k T_B|\vq_0=\vq} \\
     & \leq \sum_{k=1}^{\infty} k T_B \widetilde{\P}_{{\pi}}\pbr{\widetilde{\tau}^\pi_{\vq^*}>(k-1)T_B|\vq_0=\vq}
    \end{align}
    Note that \Cref{asspn:assump3} also holds true in the context of Markov chain $\widetilde{\vq}_k$. Thus,
    \begin{align}
    \E_{{\pi}}\sbr{\widetilde{\tau}^\pi_{\vq^*}} &\leq \sum_{k=1}^{\infty} k T_B(1-p_B)^{k-1} = \frac{T_B}{p^2_B}.
    \end{align}
    Since the bound is independent of policy $\pi\in\Pi$ and state $\vq\in B$, we obtain
    \begin{equation}
        \max_{\substack{\pi\in\Pi \\ \vq\in B}} \E_\pi\sbr{ \sum_{k=0}^{\tau^\pi_{\vq^*}}\mathbb{I}\pbr{\vq_k\in B}\Big|\vq_0=\vq} \leq \frac{T_B}{p^2_B}.
    \end{equation}
\end{proof}
Combining \Cref{lem:bddhittingtime} with \Cref{lem:upperbdoverall}, we get the following policy independent upper bound on the value function in \Cref{eq:valfunc_rec}.
\begin{equation}
    V_\pi(\vq) \leq \frac{2}{\epsilon}f^2(\vq) + (K+C_B)\pbr{\frac{T_B}{p^2_B}}
\end{equation}

\Cref{lem:bddhittingtime} can be leveraged to further obtain a policy independent upper bound on the value function as below.


\begin{restatable}{lemma}{valuefuncbound}
\label{lem:valfunc}
    Let $T_B, p_B$ be policy independent constants that satisfy Assumption~\ref{asspn:assump3} and $g,\epsilon$ be policy independent constants that satisfy Assumptions~\ref{asspn:asump2}. Then, the value function $V_\pi(\vq)$ is lower bounded $\forall \vq\in\S$ and for all policies $\pi\in\Pi$  as follows:
        \begin{equation}
         V_\pi(\vq) \geq - \frac{g}{\epsilon}{\frac{T_B}{p_B^2}}
        \end{equation}  
\end{restatable}


\begin{proof}
Recall the definition of the state value function $V_\pi(\vq)$ in Equation~\ref{eq:valfunc_rec}. Consider any state $\vq\in \S$ and policy $\pi\in\Pi,$ such that $\tau^\pi_{\vq^*}$ represents the time to hit state ${\vq^*}$ when starting at $\vq$. Then,
    \begin{align}
        V_\pi(\vq) &= \E_\pi\sbr{\sum_{k=0}^{\tau^\pi_{\vq^*}-1}\pbr{c_\pi(\vq_k)-\E_\pi\sbr{c_\pi(\vq)}} \Big|\vq_0=\vq} \\
        &= \E_\pi\sbr{\sum_{k=0}^{\tau^\pi_{\vq^*}-1}\pbr{\Tilde{c}_\pi(\vq_k)}(\mathbb{I}(\vq_k\in B)+\mathbb{I}(\vq_k\in B^{\mathsf{c}}))\Big|\vq_0=\vq}.
    \end{align}
From definition of $B$ in \Cref{eq: set_B}, we know that $\Tilde{c}_\pi(\vq)\geq 0$ when $\vq\in B^\mathsf{c}$. Hence, 
\begin{align}
    V_\pi(\vq) &\geq \E_\pi\sbr{\sum_{k=0}^{\tau^\pi_{\vq^*}-1}\pbr{\Tilde{c}_\pi(\vq_k)}(\mathbb{I}(\vq_k\in B))\Big|\vq_0=\vq}.
\end{align}
Since the instantaneous costs are non negative,
\begin{align}
    V_\pi(\vq) &\geq \E_\pi\sbr{\sum_{k=0}^{\tau^\pi_{\vq^*}-1}{-J_\pi}(\mathbb{I}(\vq_k\in B))\Big|\vq_0=\vq}.
\end{align}
From \Cref{lem:avgcost},
\begin{align}
    V_\pi(\vq) &\geq -\frac{g}{\epsilon}\E_\pi\sbr{\sum_{k=0}^{\tau^\pi_{\vq^*}-1}(\mathbb{I}(\vq_k\in B))\Big|\vq_0=\vq}.
\end{align}
From \Cref{lem:bddhittingtime}, we obtain the result,
\begin{align}
    V_\pi(\vq) &\geq -\frac{g}{\epsilon}\frac{T_B}{p^2_B}
\end{align}
\end{proof}

\subsubsection{Policy Independent Bounds on the State-Action Value Function}
\label{appdx:Qfunc}
In order to obtain policy independent bounds on the estimate $\wQ_\pi$ of the state action value function associated with some policy $\pi$, it is necessary to first obtain bounds on the exact state action value function $Q_\pi.$ The following lemma provides with state-dependent, policy-independent bounds on the state action value function $Q.$
\begin{lemma}
\label{lem:bddQ}
    There exists constant $g_1>0$, such that under Assumptions~\ref{asspn:asump1},\ref{asspn:asump2} and \ref{asspn:assump3}, the state action value function $Q_\pi$ for all policies $\pi\in\Pi$ and forall $\vq\in\S$ satisfies:
    \begin{equation}
        \left|Q_\pi(\vq,a)-Q_\pi(
        \vq,a')\right|\leq \frac{2}{\epsilon}f^2(\vq)+\frac{4D}{\epsilon}f(\vq)+\overline{c}(\vq) + g_1 \qquad a,a'\in \A
    \end{equation}
    where $\epsilon>0$ is the drift parameter and $g_1 = \frac{2D^2}{\epsilon}+ (K+C_B)\pbr{1+\frac{T_B}{p^2_B}} +\frac{g}{\epsilon}\pbr{\frac{T_B}{p_B^2}}$.
\end{lemma}
\begin{proof}
Recall the Poisson Equation~\eqref{eq:Qfunc} corresponding to the state action value function $Q_\pi$:
\begin{equation}
    Q_\pi(\vq,a) = c(\vq,a) + \E_{\vqp\sim \P(\cdot|\vq,a)}V_\pi(\vqp) - J_\pi
\end{equation}
For any pair of actions $a,a'\in\A$
    \begin{align}
        Q_\pi(\vq,a)-Q_\pi(\vq,a') &= c(\vq,a) - c(\vq,a')+ \E_{\vqp\sim \P(\cdot|\vq,a)}V_\pi(\vqp) - \E_{\vq''\sim \P(\cdot|\vq,a')}V_\pi(\vq'')\nonumber \\
        & \leq \overline{c}(\vq) + \E_{\vqp\sim \P(\cdot|\vq,a)}\pbr{\frac{2}{\epsilon}f^2(\vqp)+(K+C_B)\pbr{1+\frac{T_B}{p^2_B}}} \nonumber \\ & + \E_{\vq''\sim \P(\cdot|\vq,a')}\pbr{\frac{g}{\epsilon}\pbr{\frac{T_B}{p_B^2}}}.
    \end{align}
where the last inequality follows from Lemma~\ref{lem:valfunc}.
Recall from~\Cref{asspn:asump2}, \Cref{eq:bdddiff}, we know that $f(\vq')\leq f(\vq)+D$, for all $\vq':\P_\pi(\vq'|\vq)>0,$ for any policy $\pi.$

Hence we obtain, 
\begin{equation}
    Q_\pi(\vq,a)-Q_\pi(\vq,a') \leq \frac{2}{\epsilon}f^2(\vq)+\frac{4D}{\epsilon}f(\vq)+\overline{c}(\vq)+\frac{2D^2}{\epsilon}+ (K+C_B)\pbr{1+\frac{T_B}{p^2_B}} +\frac{g}{\epsilon}\pbr{\frac{T_B}{p_B^2}}
\end{equation}
Let $g_1 = \frac{2D^2}{\epsilon}+ (K+C_B)\pbr{1+\frac{T_B}{p^2_B}} +\frac{g}{\epsilon}\pbr{\frac{T_B}{p_B^2}}$, then we obtain the following:
\begin{equation}
    Q_\pi(\vq,a)-Q_\pi(\vq,a') \leq \frac{2}{\epsilon}f^2(\vq)+\frac{4D}{\epsilon}f(\vq)+\overline{c}(\vq) + g_1
\end{equation}

Since the above inequality is true for all $a,a'\in\A$, 
\begin{equation}
    \left|Q_\pi(\vq,a)-Q_\pi(\vq,a')\right| \leq \frac{2}{\epsilon}f^2(\vq)+\frac{4D}{\epsilon}f(\vq)+\overline{c}(\vq) + g_1
\end{equation}
\end{proof}

\subsection{Proof of Step 3}
The previous step provided us with bounds over the exact state action value function. Here we incorporate the policy evaluation error 
to obtain bounds over the state action value function estimate. 

\begin{restatable}{lemma}{Qestbound}
\label{lem:Mq}
    For all states $\vq\in\S$, all pairs of actions $a, a'\in\A$, and all policies $\pi$, it is true that,
    \begin{align*}
        \left|\wQ_\pi(\vq, a) - \wQ_\pi(\vq, a')\right| \le 2\delta(\vq) + \frac{2}{\epsilon}f^2(\vq)+\frac{4D}{\epsilon}f(\vq)+\overline{c}(\vq) + g_1
    \end{align*}
    where $\wQ_\pi(\vq,a)$ is the estimate of $Q_\pi(\vq,a)$ such that $\delta(\vq):= \left|\wQ_\pi(\vq, a) - Q_\pi(\vq, a)\right|, \forall a\in\A.$
\end{restatable}

\begin{proof}
    \begin{align}
        \left|\wQ_\pi(\vq, a) - \wQ_\pi(\vq, a')\right| &=  \left|\wQ_\pi(\vq, a) - Q_\pi(\vq,a) + Q_\pi(\vq,a) - \wQ_\pi(\vq, a') + Q_\pi(\vq,a') - Q_\pi(\vq,a')\right| \\
        & \leq \left|\wQ_\pi(\vq, a) - Q_\pi(\vq,a)\right| + \left|\wQ_\pi(\vq, a') - Q_\pi(\vq,a')\right| \\ & + \left|Q_\pi(\vq, a) - Q_\pi(\vq,a')\right|
    \end{align}
    From Equation~\ref{eq:pe_error} and Lemma~\ref{lem:bddQ}, it follows that, 
    \begin{equation}
        \left|\wQ_\pi(\vq, a) - \wQ_\pi(\vq, a')\right| \leq 2\delta(\vq) + \frac{2}{\epsilon}f^2(\vq)+\frac{4D}{\epsilon}f(\vq)+\overline{c}(\vq) + g_1
    \end{equation}
\end{proof}

\subsection{Proof of Main theorem (Step 4)}
\label{appendx:maintheorem}

The proof requires utilizing the performance difference lemma to establish a connection between the difference in average cost associated with a policy \(\pi\) and the optimal average cost in terms of the state-action value function \(Q_\pi\).
\begin{lemma}
    \label{lem: pdf}
    Let $J_\pi$ and $J_{\pi'}$ be the expected infinite horizon average cost associated with policies $\pi$ and $\pi'$ respectively. Let $d_\pi$ be the stationary distribution over state space $\S$ associated with $\P_\pi.$ Then it is true that, 
    \begin{equation}
        J_\pi-J_{\pi'} = \sum_{\vq\in\S}d_\pi(\vq)\sbr{Q_{\pi'}(\vq,\pi(\vq))-Q_{\pi'}(\vq,\pi'(\vq))}
    \end{equation}
    where $Q_{\pi'}(\vq,\pi(\vq))=\sum_{a\in\A}\pi(a|\vq)Q_{\pi'}(\vq,a)$ and $Q_{\pi'}(\vq,\pi'(\vq))=V_{\pi'}(\vq).$
\end{lemma}
\begin{proof}
    The proof can be found in \cite{cao1999single}.
\end{proof}

We restate the theorem for convenience. 
\maintheorem*
\begin{proof}
    Let $J^*$ be the optimal average cost.
    Let $\pi^*\in\Pi$ be the optimal policy. For any policy $\pi\in\Pi$, performance difference lemma provides the following, 
    \begin{align}
         J_\pi -J^* & = - \E_{\vq\sim d_{\pi^*}}\sbr{Q_\pi\pbr{\vq,\pi^*(\vq)}-Q_\pi\pbr{\vq,\pi(\vq)}} \\
        &= - \E_{\vq\sim d_{\pi^*}}\Big[Q_\pi\pbr{\vq,\pi^*(\vq)}- \wQ_\pi\pbr{\vq,\pi^*(\vq)}+\wQ_\pi\pbr{\vq,\pi^*(\vq)}-Q_\pi\pbr{\vq,\pi(\vq)} \\ &\qquad \qquad \qquad +\wQ_\pi\pbr{\vq,\pi(\vq)}-\wQ_\pi\pbr{\vq,\pi(\vq)}\Big] \\
        &\leq \E_{\vq\sim d_{\pi^*}}\sbr{\left|Q_\pi\pbr{\vq,\pi^*(\vq)}-\wQ_\pi\pbr{\vq,\pi^*(\vq)}\right|} + \E_{\vq\sim d_{\pi^*}}\sbr{\left|Q_\pi\pbr{\vq,\pi(\vq)}-\wQ_\pi\pbr{\vq,\pi(\vq)}\right|} \\ &
        \qquad \qquad \qquad +\E_{\vq\sim d_{\pi^*}}\sbr{\wQ_\pi\pbr{\vq,\pi(\vq)} - \wQ_\pi\pbr{\vq,\pi^*(\vq)}}
    \end{align}
    From Equation~\ref{eq:pe_error}, we know that $\E\sbr{\left|Q_\pi\pbr{\vq,a}-\wQ_\pi\pbr{\vq,a}\right|}\leq \kappa(\vq)$. Hence we obtain the following:
    \begin{equation}
         \E\pbr{J_\pi - J^*} \leq 2\E_{\vq\sim d_{\pi^*}}\pbr{\kappa(\vq)} +\E_{\vq\sim d_{\pi^*}}\sbr{\E\pbr{\wQ_\pi\pbr{\vq,\pi(\vq)} - \wQ_\pi\pbr{\vq,\pi^*(\vq)}}}
    \end{equation}
    The total expected regret across time horizon $T$ can be expressed by summing the above inequality as follows,
    \begin{equation}
        \sum_{k=1}^T \E\sbr{J_{\pi_k} - J^*} \leq 2T\E_{\vq\sim d_{\pi^*}}\pbr{\kappa(\vq)} + \E_{\vq\sim d_{\pi^*}}\sbr{\E\pbr{\sum_{k=1}^T\pbr{\wQ_{\pi_k}\pbr{\vq,\pi_k(\vq)} - \wQ_{\pi_k}\pbr{\vq,\pi^*(\vq)}}}}
    \end{equation}
    where $\pi_k$ are policy iterates obtained through the NPG policy update below:
    \begin{equation}
        \pi_k(a|\vq) = \frac{\pi_{k-1}(a|\vq)\exp{\pbr{-\eta_\vq \wQ_{\pi_{k-1}}(\vq,a)}}}{\sum_{l\in\A} \pi_{k-1}(l|\vq)\exp{\pbr{-\eta_\vq \wQ_{\pi_{k-1}}(\vq,l)}}}
    \end{equation}
    The above update is performed for all $\vq$ and $a\in\A$. Let the update parameter $\eta_\vq = \sqrt{\frac{8\log|\A|}{T}}\frac{1}{M_\vq}$, where $M_\vq = 2\delta(\vq) + \frac{2}{\epsilon}f^2(\vq)+\frac{4D}{\epsilon}f(\vq)+\overline{c}(\vq) + g_1.$ Then from Theorem \ref{th:weightavg}, it follows that,
    \begin{align}
        \sum_{k=1}^T \E\sbr{J_{\pi_k} - J^*} &\leq 2T\E_{\vq\sim d_{\pi^*}}\pbr{\kappa(\vq)}  + \E_{\vq\sim d_{\pi^*}}\sbr{\sqrt{\frac{T\log|\A|}{2}}\E\sbr{M_\vq}} \\
        & = 2T\E_{\vq\sim d_{\pi^*}}\pbr{\kappa(\vq)} + \sqrt{\frac{T\log|\A|}{2}} \pbr{\E_{\vq\sim d_{\pi^*}}\pbr{2\kappa(\vq) + \frac{2}{\epsilon}f^2(\vq)+\frac{4D}{\epsilon}f(\vq)+\overline{c}(\vq)} + g_1}\\
        & \leq 2T\E_{\vq\sim d_{\pi^*}}\pbr{\kappa(\vq)} + \sqrt{\frac{T\log|\A|}{2}} \pbr{\E_{\vq\sim d_{\pi^*}}\pbr{2\kappa(\vq) + \frac{2}{\epsilon}f^2(\vq)+\frac{4D}{\epsilon}f(\vq)} + \frac{g}{\epsilon} + g_1}
    \end{align}
    where the last inequality follows from the fact that $\overline{c}(\vq)\leq\frac{g}{\epsilon}$ in \Cref{lem:avgcost}.

    Let $\beta:=\E_{\vq\sim d_{\pi^*}}\sbr{\kappa(\vq)}$ be defined. From \Cref{lem:hajek}, it is known that moments of $f(\vq)$ exist. Let $\beta_1 = \frac{4D}{\epsilon}\E_{\vq\sim d_{\pi^*}}\sbr{f(\vq)}$ and $\beta_2 = \frac{2}{\epsilon}\E_{\vq\sim d_{\pi^*}}\sbr{f^2(\vq)}$. Hence, we obtain, 
    \begin{equation}
        \sum_{k=1}^T \E\sbr{J_{\pi_k} - J^*} \leq 2\beta T + \sqrt{T}\pbr{\sqrt{\frac{\log|\A|}{2}}\pbr{2\beta+\beta_1+\beta_2+\frac{g}{\epsilon}+g_1}}
    \end{equation}
    Setting $c'=\sqrt{\frac{\log|\A|}{2}}\pbr{2\beta+\beta_1+\beta_2+\frac{g}{\epsilon}+g_1}$ and $c''=2\beta$ yields the result in the theorem.
\end{proof}




\end{document}